\documentclass[manuscript]{acmart}

\AtBeginDocument{%
  \providecommand\BibTeX{{%
    \normalfont B\kern-0.5em{\scshape i\kern-0.25em b}\kern-0.8em\TeX}}}

\setcopyright{acmlicensed}
\copyrightyear{2024}
\acmYear{2024}
\acmDOI{XXXXXXX.XXXXXXX}

\acmConference[ACM FAccT '24]{}{ June 3–6, 2024}{Rio de Janeiro, Brazil}
\acmISBN{979-8-4007-0450-5/24/06.}

 \usepackage{subcaption}
\floatstyle{ruled}
\newfloat{listing}{tb}{lst}{}
\floatname{listing}{Listing}
\usepackage{algorithm}
\usepackage{algpseudocode}
 \usepackage{amsthm}      
 \usepackage{mathtools}  % amsmath with extensions
\usepackage{dsfont}
\usepackage{lipsum}
\usepackage{mathtools}
\usepackage{cuted}
\newtheorem{example}{Example}
\newtheorem{theorem}{Theorem}
\newtheorem{lemma}{Lemma}
\newcommand{\argmax}{\operatornamewithlimits{argmax}}

\newcommand{\pr}{\mathbb{P}}

\newcommand{\vtheta}{{\theta}} %by Catuscia: I redefined this command to be consistent with the notation of Daniele
\begin{document}
\pagenumbering{gobble}
\title{BaBE: Enhancing Fairness via Estimation of Explaining Variables}

% \author{Ruta Binkyte}
% \affiliation{%

\author{Ruta Binkyte}
\affiliation{%
  \institution{CISPA Helmholtz Center for Information Security}
  \streetaddress{}
  \city{}
  \country{Germany}}
\email{ruta.binkyte-sadauskiene@cispa.de}

\author{Daniele Gorla}
\affiliation{%
  \institution{Università di Roma ``La Sapienza''}
  \city{}
  \country{Italy}
}

\author{Catuscia Palamidessi}
\affiliation{%
  \institution{Inria Saclay and École Polytechnique}
  \city{}
  \country{France}
}

% \affiliation{%

% \affiliation{%

% \affiliation{%

% \affiliation{%

% \affiliation{%

\renewcommand{\shortauthors}{Binkyte, et al.}

\begin{abstract}
  We consider the problem of unfair discrimination between two groups and propose a pre-processing method to achieve fairness. Corrective methods like statistical parity
usually lead to bad accuracy and do not really achieve fairness in situations where there is a correlation between the  sensitive attribute $S$ and the legitimate attribute $E$ (explanatory variable) that should determine the decision.
%, like, for example, in the healthcare domain. 
To overcome these drawbacks, other notions of fairness have been proposed, in particular, conditional statistical parity and equal opportunity. 
However, $E$ is often not directly observable in the data. We may observe some other variable $Z$ representing $E$, but the problem is that $Z$ may also be affected by $S$, hence $Z$ itself can be biased. To deal with this problem, we propose BaBE (Bayesian Bias Elimination), an approach based on a combination of Bayes inference and the Expectation-Maximization method, to estimate the most likely value of $E$ for a given $Z$ for each group. The decision can then be based directly on the estimated $E$. We show, by experiments on synthetic and real data sets, that our approach provides a good level of fairness as well as high accuracy.
\end{abstract}




\acmYear{2024}\copyrightyear{2024}
\setcopyright{acmlicensed}
\acmConference[ACM FAccT '24]{ACM Conference on Fairness, Accountability, and Transparency}{June 3--6, 2024}{Rio de Janeiro, Brazil}
\acmBooktitle{ACM Conference on Fairness, Accountability, and Transparency (ACM FAccT '24), June 3--6, 2024, Rio de Janeiro, Brazil}
\acmDOI{10.1145/3630106.3659016}
\acmISBN{979-8-4007-0450-5/24/06}


\maketitle

\section{Introduction}
% \noindent An increasing number of decisions regarding the daily lives of human beings are relying on machine learning (ML) predictions, and it is, therefore, crucial to ensure that they are not only accurate but also objective and fair.

One of the first group of fairness notions proposed in literature was \emph{statistical parity} (SP) \cite{dwork2012fairness}, which enforces the probability of a positive prediction to be equal across different groups. Let the prediction and the group be represented, respectively, by the random variables $\hat{Y}$ and $S$, both of which are assumed to be binary for simplicity, and let $\hat{Y}=1$ stand for the positive prediction. Then SP is formally described by $\pr[\hat{Y}=1|S=1] \; = \; \pr[\hat{Y}=1|S=0]$, where $\pr[\cdot|\cdot]$ represents conditional probability.

However, SP has been criticized for causing loss of accuracy and for ignoring circumstances that could justify  disparity. A more refined notion is \emph{conditional statistical parity} (CSP)~\cite{kamiran_quantifying_2013}, which allows some disparity as long as it is legitimated by explaining factors. For example, a hiring decision positively biased towards Group $1$ could be justified if Group $1$ has a higher education level than Group $0$ in average. CSP is formally defined by $\pr[\hat{Y}=1|S=1,E=e]\; = \; \pr[\hat{Y}=1|S=0,E=e]$, for all $e$, where $E$ is a random variable representing the ensemble of explaining features.

The most common pre-processing approach  to achieve CSP (or an approximation of it) consists in editing the  label $Y$ (decision) in the training data,  according to some heuristic, so to ensure that the number of samples with $Y=1$, $S=1$, and $E=e$ are approximately the same number as those with  $Y=1$, $S=0$, and $E=e$. 
One problem, however, is that often $E$ is not directly observable in the data. Usually, we can observe some other variable $Z$ that is representative of $E$, but the problem is that $Z$ may be also influenced by the sensitive attribute $S$, hence $Z$ itself can be biased. We illustrate this scenario with the following examples.

%%%%
%%%%
\begin{figure*}[t]
 \centering
  \begin{subfigure}[t]{0.45\textwidth}
 \centering
  \includegraphics[height=3cm]{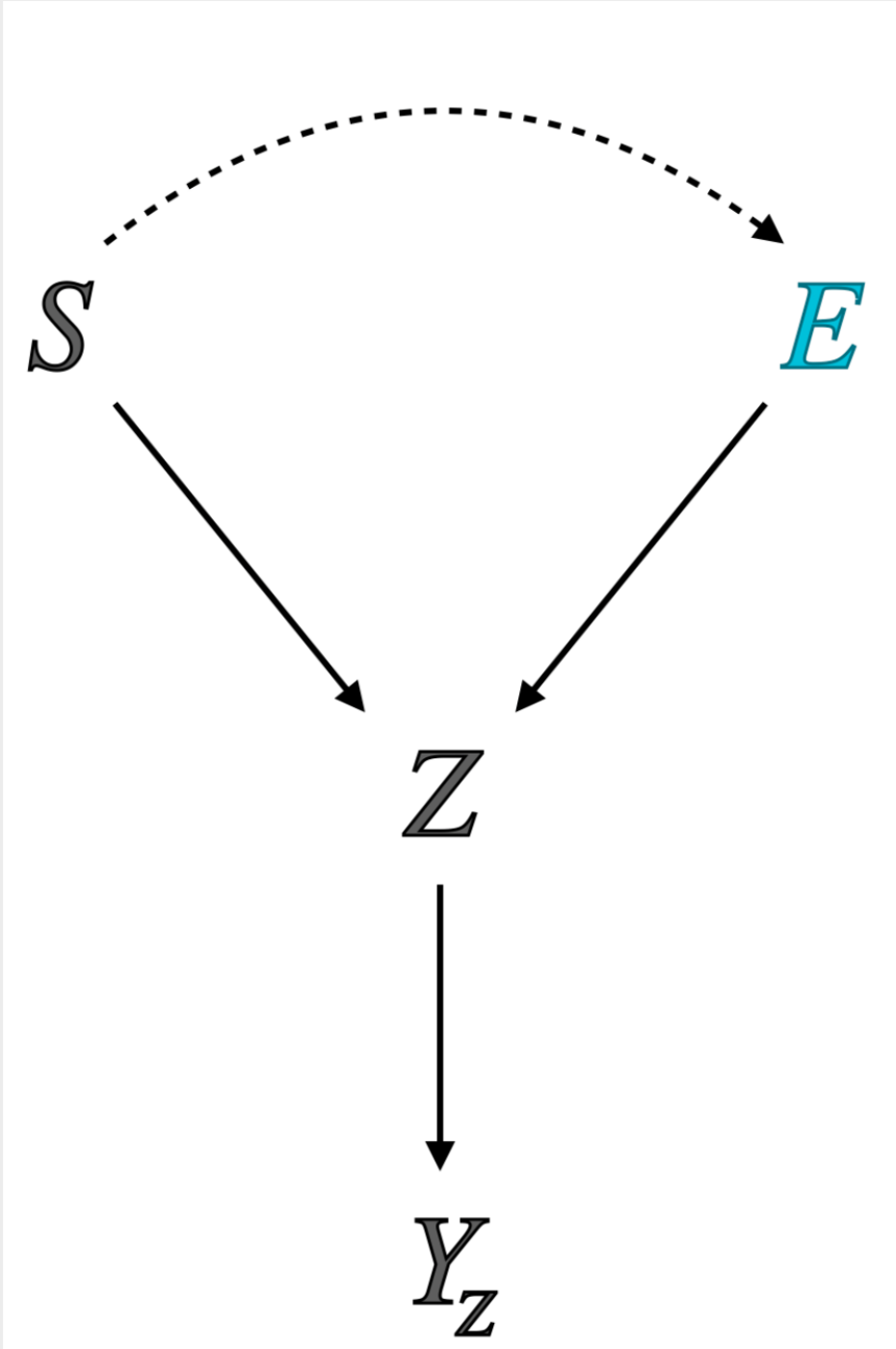}
   % \caption{Causal relation between the variables $S$, $E$,  $Z$, and the decision $Y_{Z}$ based on $Z$. }   
    % \label{subfig:models-a}
  \end{subfigure}
  %\hfill
  \begin{subfigure}[t]{0.45\textwidth}
  \centering
 \includegraphics[height=3cm]{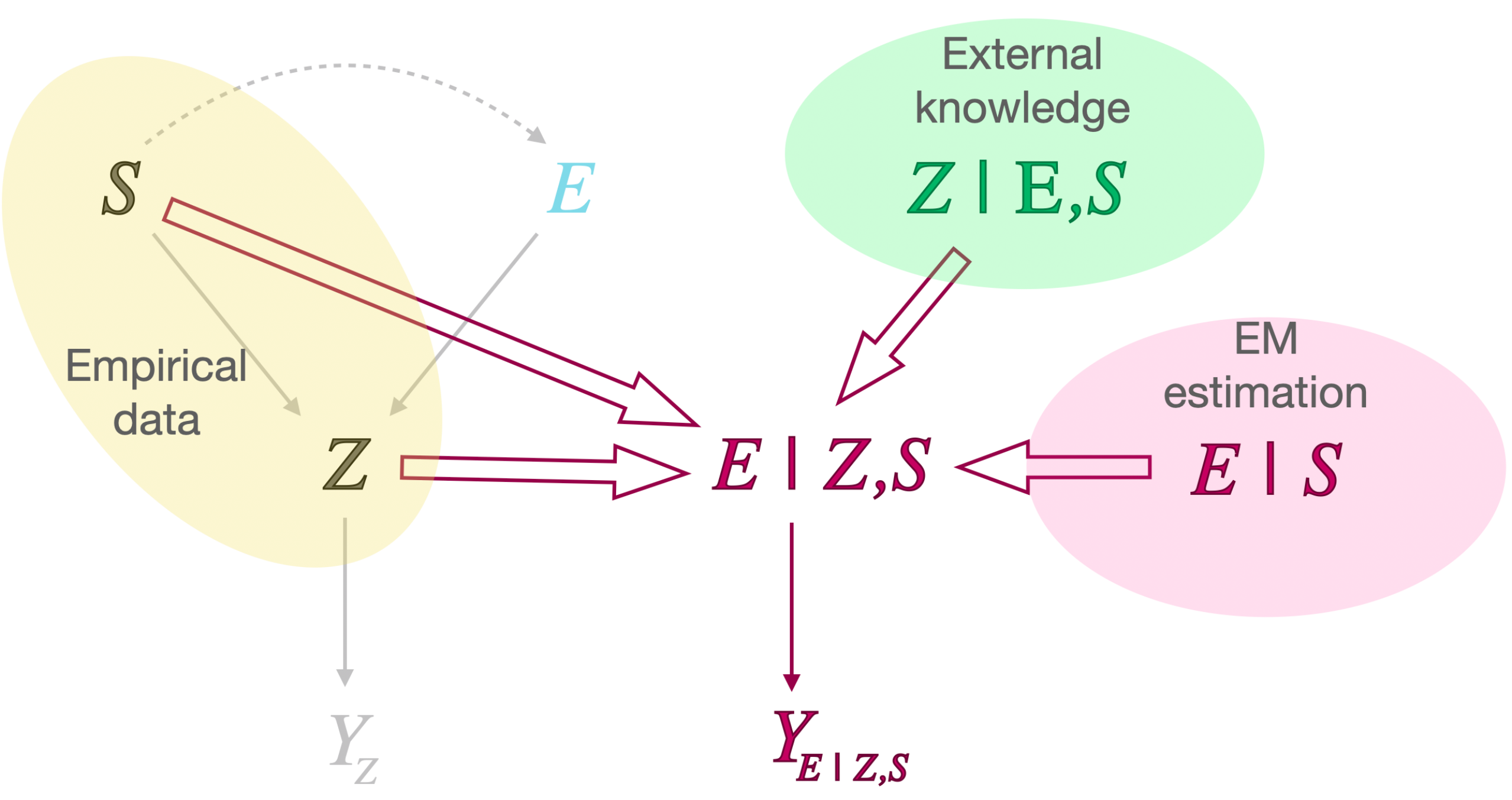}   
  % \caption{Derivation of $E|Z,S$. The decision $Y_{E}$ is then  based directly on $E$.}
   % \label{subfig:models-b}
  \end{subfigure}
\caption{Left:  illustration of the causal relation between the data. Right: illustration of our  pre-processing method.} \label{fig:mainmodels}
\end{figure*}
%%%%
%%%%

\begin{example}\label{exa:bias}
The SAT (Scholastic Assessment Test) is a standardized test widely used for college admissions in the United States
%In principle, the SAT score is supposed 
aiming at
indicating the skill level of the applicant, and therefore her potential to succeed in college. 
However, the performance at the test can be affected by other socio-economic, psychological, and cultural factors. 
% Because of these  spurious influences, the SAT practice has been routinely criticized for significant gender, race and income bias, and has been
% the object of extensive social studies. 
For instance, a recent study \cite{taketwo} points out that, on average, black students are less likely to undergo the financial burden of retaking the test than white students. This causes a racial gap in the scores,  since retaking the test usually improves the result. 
Another study \cite{testanxitetyhannon} reports that,  on average,  girls score  approximately 30 points less on SAT than boys, despite the fact that girls routinely achieve  higher grades in  school. According to \cite{testanxitetyhannon}, the cause is the higher sensitivity to stress and test anxiety among females.
\end{example}

\begin{example}\label{exa:bias2}
  Many healthcare systems in the United States rely on prediction algorithms to identify patients in need of assistance. One of the most used indicators is the individual healthcare expenses, as  they are easily available in the insurance claim data. However, healthcare spending is influenced not only by the health condition, but also by the socio-economic status. 
 A recent study \cite{obermeyer2019dissecting}  shows that typical algorithms used by these healthcare systems are negatively biased against black patients, in the sense that, for the same prediction score, black patients are in average sicker than the white ones.  According to \cite{obermeyer2019dissecting}, this is due to the bias in the healthcare spending data, since  black patients spend less on healthcare due to lower financial capabilities and lower level of trust towards the white-dominated medical system and practitioners.
\end{example}

In the above examples, the ``true skills'' and the ``true health status'', respectively, are the legitimate features $E$ (explanation) on which we should base the decision. Unfortunately $E$  is not directly observable. What we can observe, instead, is the result of the SAT test and the healthcare-related spending, respectively. These are represented by the variable   $Z$. These indicators, however, do not faithfully represent $E$, because they are  influenced also by other factors, namely the economical status (or the gender),  and the race, respectively. These are the sensitive attribute $S$.

The line of research that advocates the use of statistical parity  \cite{calders_nbayes, choi2021group,islam2022fair,louizos2015variational,madras2018fairness} adheres to the ``we are all equal'' principle~\cite{friedler2016possibility}, and makes the basic assumption that \emph{$E$ and $S$ are independent}. However, in many cases, like for instance in decisions regarding the medical treatment of genetic illnesses, race or gender could have a direct effect on the likeliness of the medical condition.
%Alternatively, the link could be indirect (mediated by other variables), but nevertheless causal and statistical parity cannot be implied. 
For example, in our second running example, the real health status is on average lower in the black population because of socio-economic factors. Hence, we allow the possibility of a 
%direct or mediated 
link between the sensitive attribute $S$ and the explaining value $E$, and aim to remove  the discrimination  introduced by the link between $S$ and $Z$. 
The method we propose to remove the discrimination works equally well whether or not there is a link between $S$ and $E$, and it does not modify this relation. 

To summarize, in the original (unfair) scenario the decision $Y$ is based on $Z$, which is influenced by both $E$ and $S$. The situation is represented in Figure~\ref{fig:mainmodels} (left). The arrow from $S$ to $Z$ represents that there is a causal relation between  $S$ and $Z$, and similarly for the other solid arrows\footnote{Note that $E$  is what in causality is called \emph{a mediator}.}, while the dashed arrow between $S$ and $E$ represents a relation that may or may not be present. 
%\emph{A collider} structure introduces a spurious link or exaggerates  the existing causal link by the additional non causal correlation. The variables that are otherwise independent become dependent when conditioning on $Z$, for example, using a SAT score or the medical expenditure for the decision. 
% In this paper  we make the  basic assumption that \emph{$E$ and $S$ are independent}.
% So, for instance, we assume that the ``true skills'' are independent from race and gender~\footnote{The validity of this assumption is of course  domain-dependent. It is reasonable to assume it in the situations like those described in Example~\ref{exa:bias}. In other cases, like for instance in decisions regarding the medical treatment of genetic illnesses, it would not be appropriate to consider the illness independent from the race or the gender.}.  
In order to take a fair decision, we would like to base the decision $Y$ only on $E$, but, as explained before, $E$ may  not be directly available. 
% In this work, we address the problem of recovering $E$ from the observable variables and from some additional knowledge about the relation between them. More precisely, t
Therefore, we need to determine what is the most likely value of $E$ for the given values of $S$ and $Z$. To this purpose, we will derive the conditional distribution of $E$ given $Z$ and $S$, i.e. $\pr[E|Z,S]$. The objective is illustrated in Figure~\ref{fig:mainmodels} (right). 

Note that $E$ can be multi-dimentional, and that we represent the effect of other possible latent variables by the randomness in the distribution of the data.

The method we propose uses a combination of the \emph{Bayes theorem} and  the \emph{Expectation-Maximization method} (EM) ~\cite{Dempster:77:RSS}, a powerful statistical technique to estimate unobservable variables as the  maximum likelihood parameters of empirical  data observations. 
We  call our method BaBE, for \emph{Bayesian Bias Elimination}. 

% %%%%
% %%%%
% \begin{figure}[t]
% \centering
%   \begin{subfigure}[t]{0.38\linewidth}
% \centering
%  \includegraphics[height=3cm]{Graphs/causal}
%    \caption{Causal relation between the variables $S$, $E$,  $Z$, and the decision $Y_{Z}$ based on $Z$. }     \label{subfig:models-a}
%   \end{subfigure}
%   \hfill
%   \begin{subfigure}[t]{0.55\linewidth}
%   \centering
%  \includegraphics[height=3cm]{Graphs/IBU-fairness}   
%   \caption{Derivation of $E|Z,S$. The decision $Y_{E}$ is then  based directly on $E$.}
%   \label{subfig:models-b}
%   \end{subfigure}
% \caption{(a) illustrates the causal relation between the data,  (b) illustrates our   pre-processing method.}  
% \label{fig:models}
% \end{figure}
% %%%%
% %%%%

BaBE relies on some additional knowledge, 
namely an estimation of  the conditional distribution of $Z$ given $S$ and $E$, i.e., $\pr[Z|E,S]$.  This estimation can be obtained by collecting additional data. For instance, for   Example~\ref{exa:bias2}, we could use the richer set of biomarkers, like in \cite{obermeyer2019dissecting}. 
Alternatively, it can be produced by studies or 
experiments in a controlled environment. For instance, for Example~\ref{exa:bias}, we could assess skills in some subjects by in-depth examinations, and derive statistics about their SAT performance both at the first attempt and after several retakes. Another possibility is to collect data on the subsequent performance of the students that have been accepted, and of those who have not been accepted in the school in question but have been accepted in another school.

One obvious question that may arise is: what are the advantages of deriving $\pr[Z|E,S]$, rather than directly $\pr[E|Z,S]$, from the additional data? (The derivation of the latter from the former is the essence of our proposal.) We argue that, while in general 
there may not be any advantage, there are real-life situations in which $\pr[Z|E,S]$ is more ``universal'' than $\pr[E|Z,S]$, in the sense that the first does not depend on the distribution of $E|S$ ($E$ given $S$), while the latter does. As a consequence, the knowledge of the first can be re-used in different contexts, while the latter cannot. 
One typical example is the study of 
symptoms ($Z$) induced by certain diseases ($E$), which may also depend on the gender or other characteristics such as ethnicity, age, etc. ($S$): $\pr[Z|E,S]$ can be statistically estimated from medical data $D$ collected by some hospitals, and it is reasonable to assume that it does not depend on the distribution of $E|S$, which, in contrast, can vary greatly depending on the geographical area, on the social context, etc.  Also $\pr[E|Z,S]$ could be estimated from $D$, but it may depend on $E|S$. For example, in towns that are very polluted (area $A_1$), the risk that coughing (symptom, $Z$) indicates lung cancer rather than a simple cold (diseases, $E$) may be much higher than in the (less polluted) area $A_2$ where the data $D$ were collected. The idea of BaBE to predict diseases in $A_1$ is to estimate $\pr[Z|E,S]$ in the area where complete data  (including $E$) are available, in this case, area $A_2$, assuming that the same $\pr[Z|E,S]$ is valid also in $A_1$. Then, we estimate the empirical probability (frequency)  $\pr[Z|S]$ in $A_1$. Subsequently, using the above $\pr[Z|E,S]$ and $\pr[Z|S]$, the BaBE method allows us to derive $\pr[E|S]$ in $A_1$. Finally, by applying the Bayes theorem to the above probabilities, we  derive $\pr[E|Z,S]$ in $A_1$.

% The estimation of  $\pr[E|Z,S]$ in this way would be less accurate because not all patients affected by symptoms necessarily enter the medical system. 

% Note that deriving $E|Z,S$ directly in this way would be less efficient, because in general we have less control over $Z$.\footnote{In order to estimate $\pr[E|Z,S]$ we need to consider all possible degrees of skill, and to select a number of subjects for each degree, balanced across the different degrees. Such selection can be carried on efficiently by considering the performance of potential candidates in previous schools, for example, and then double checking their skill degree with in-depth examinations. 
% % Doing the same thing for $P_{E|Z,S}$ is more difficult, as the SAT scores are more aleatory and less predictable.
% } 

We note that the scenario we are considering is the same as that of machine learning (ML). Indeed, in machine learning, we assume the existence of a dataset (for instance, historical data), i.e., a representation of the joint distribution $\pr[E,Z,S]$. In the case of ML, we typically derive directly the prediction of $E$ for a given value of $Z$ and $S$. However,  
it may happen that 
$\pr[E|Z,S]$ depends on the distribution of $E|S$, which can vary greatly depending on the context. 
The effect of the distribution shift is  a well known problem in ML, impeding the deployment of the model in  populations that are different from the one in the training data~\cite{quinonero2008dataset, ovadia2019can}.

% the proportion of students with high potential in Paris may not be the same as in London; 
% As another example, the racial health gap can be different between Europe and the United States. 
In contrast, $\pr[Z|E,S]$ may be more ``universal'', and this is exactly the case in which our BaBE method is applicable, also in case of a distribution shift (on $E$).  In this case, it is convenient to invest in the estimation of $\pr[Z|E,S]$, which can be done once and then transferred to different contexts. 
Indeed, one advantage of our approach is that it allows the transfer of causal knowledge. Namely, once we learn the relation $\pr[Z|E,S]$, the method can be applied to a population with different proportions, i.e. different  $\pr[E|S]$ (but the same $\pr[Z|E,S]$). 
For more discussion about this point, we refer to  \cite{Bareinboim2013CausalTW,Bareinboim_2013,bareinboim2014transportability,pearl2022external,bengio_2021}. 
Another case in which our method presents an advantage over ML is when it is possible to estimate causal prior knowledge from experimental data, which is typically small. Machine learning algorithms need large data sets to achieve a good performance, whereas Bayesian statistics can be suitable also for small sample sizes~\cite{heerwegh2014small,mcneish2016using}.

% To obtain $\pr[E|Z,S]$ from $\pr[Z|E,S]$ we can then apply the conditional Bayes theorem and compute 
% %\begin{align*}
% %&\pr[E=e|Z=z,S=s] \;=\;\\
% %&\frac{\pr[Z=z|E=e,S=s] \cdot \pr[E=e|S=s] }{\pr[Z=z|S=s]}
% % \;=\;\frac{\pr[Z=z|S=s,E=e] \cdot \pr[E=e|S=s] }{\pr[Z=z|S=s]}.
% %\end{align*}
% $\pr[E=e|Z=z,S=s]$ as the ratio between
% $\pr[Z=z|E=e,S=s] \cdot \pr[E=e|S=s]$ and $\pr[Z=z|S=s]$. However, while $\pr[Z=z|S=s]$ can be estimated from the data at our disposal, the problem is that the prior $\pr[E=e|S=s]$ is unknown. 
% % $\pr[E]$ (which is the same as $P_{E|S}$ due to the assumption of independence between $E$ and $S$), 
% In order to estimate it, we will use the \emph{Expectation-Maximization method} (EM) ~\cite{Dempster:77:RSS}, a powerful statistical technique to estimate latent variables as the  maximum likelihood parameters of empirical  data observations. 
% Since  our method relies on the Bayes theorem, we  call it BaBE, for \emph{Bayesian Bias Elimination}. 

Once  $\pr[E|Z,S]$ is estimated, we  pre-process the training data  by assigning a decision $Y$ based on the most likely value $e$ of $E$, for given values of $S$ and $Z$. If $e$ does not have enough probability mass, however, we may not achieve CSP, or even a good approximation of it. 
In such case, we can base the decision on a threshold for the estimated $E$, aiming at achieving equal opportunity (EO) \cite{hardt2016equality} instead, that we regard as a relaxation of CSP. 
Formally, EO is descibed as follows:
$\pr[\hat{Y}=1| Y=1,S=1] \; = \; 
 \pr[\hat{Y}=1|Y=1,S=0]$,
where $Y$ represents the ``true decision'', i.e., the decision based on a threshold for the real value of $E$. 

We validate our method by performing experiments,\footnote{The software used for implementing our approach and for performing the experiments is available at
%\href{https://github.com/BaBE-Algorithm/BaBE}{https://github.com/BaBE-Algorithm/BaBE}.}
{\tt https://github.com/BaBE-Algorithm/BaBE}.}
both on synthetic datasets and on the real ‘The National Health and Nutrition Examination Survey' ($NHANES$) data set \cite{NHANES}, featuring biological and chronological age of the patients.
In both cases, we obtain a very good estimation of $\pr[E|S]$, and 
we achieve a  good level of both accuracy and fairness. 

% In the case of the health care experiment, our results are in line with the analysis of  \cite{obermeyer2019dissecting}. When the  estimated distributions $\hat \pr[E|Z=z,S=s]$ are strongly unimodal, we also obtain a very good accuracy in estimating $E$. In any case, the prediction of $Y$ (based on the estimation of $E$) has high accuracy, far better than the methods based on statistical parity.

% As a final remark, we note that an additional advantage of our method is that the resulting model is  highly explainable, as we force the decision to be based on $E$. Explainability is a key desideratum for any model that is to be used as help to take  real-world decisions. 

% An instance of the EM method  has been used also in the field of Privacy Enhancing Technologies (PETs),    
% where it is known as \emph{Iterative Bayesian Update} (BaBE), see for instance \cite{agrawal2001design,elsalamouny2019convergence}. 

Summarizing, our contributions are as follows:
%\subsection{Contributions}
\begin{itemize}

\item We propose an approach to estimate the  distribution of an explaining variable $E$, using  the Expectation-Maximization method (EM). To the best of our knowledge, this is the first time that EM is used to achieve fairness without assuming the independence between $E$ and the sensitive attribute $S$.  
From the above, we then derive an estimation of $\pr[E|Z,S]$.

\item Using the estimation of $\pr[E|Z,S]$, we  show how to to estimate the values of $E$ and $Y$ for each value of $Z$ and $S$. These estimations are then used to \emph{pre-process the data in order to achieve CSP and/or EO}. 
% {\color{blue} Maybe we can add: We learn the relation $P_{E|S}$ of the target distribution, this allows to transfer causal knowledge to the population with different proportions than the one where experiments were performed.}

\item We show experimentally that our proposal outperforms other approaches for fairness, in terms of CSP, EO, accuracy, and other metrics for fairness and precision of the estimations. 

\end{itemize}

\paragraph{Related Work}
The notion of fairness that we consider in this work was introduced in \cite{kamiran_quantifying_2013} 
%under the name \emph{conditional non-discrimination}, 
and it is known nowadays as 
\emph{conditional statistical parity} (CSP)
%following the terminology introduced in
\cite{corbett2017algorithmic}. 
In \cite{kamiran_quantifying_2013},  CSP is achieved through data pre-processing, by applying 
\emph{local massaging} or \emph{local preferential sampling} techniques. However, the authors consider only an explanatory variable $E$ which is part of the data at the time of deployment of their method. 

Note that our  $Z$, although observable,  cannot be considered as an  explanatory variable, because we are assuming it is influenced by the sensitive attribute in a way that would make it unfair to base the decision on $Z$. 
%\iffull
To better understand the difference, consider one of the main examples used in \cite{kamiran_quantifying_2013} to  illustrate the idea, which  is a kind of \emph{Berkeley admission anomaly},  an instance of the \emph{Simpson paradox}~\cite{glymour2016causal}.  In this example, the admittance in a certain university looks biased against females, but the disparity can actually be explained by the fact that female students tend to choose a more selective program. In this case, the explanatory variable is a mediator (the choice of the program), 
and it is assumed to be legitimate as a cause for disparity. By contrast, in our example the observed score is considered to be influenced by social discrimination, hence it cannot be directly used as an explanatory variable. 
%\fi

The work closest to ours is \cite{calders_nbayes}, where there is a model containing a latent variable whose distribution is discovered through the Expectation Maximization method. However, in \cite{calders_nbayes} the notion of fairness considered is \emph{statistical parity} (SP). 
Using SP as a constraint 
(thus applying a sort of \emph{self-fulfilling prophecy} approach) and other constraints such as the preservation of the total ratio of positive decisions, the authors determine what the distribution 
$\pr[Z|E,S]$ should be, they distribute the probability mass uniformly on all attributes,  and they finally apply the EM method to determine the fair labels. 
In contrast, we are aiming at discovering what is the most probable value of  $E$ for each combination of values of the other attributes ($S$ and $Z$), so as to take a fair decision based on $E$, considered as the explanatory variable. We do not require statistical parity, nor do we assume a uniform distribution on all attributes. Instead, we use external knowledge 
%(social studies, etc.) 
as prior knowledge for applying the EM method. 
 Another difference 
 %between \cite{calders_nbayes}  and our work 
 is that they optimize  accuracy  with respect to the observed biased labels, whereas we  consider accuracy towards the true fair label dependent on $E$, considered as the actual attribute on which the decision should be made.
 
Similar in spirit to \cite{calders_nbayes}, \cite{louizos2015variational} tries to discover the latent variable which is maximally informative about the decision,  while  minimizing the correlation with the sensitive attribute (statistical disparity); this is done via a deep learning technique.
%that they call \emph{variational fair auto-encoding}. 
Also \cite{kusner2018counterfactual,louizos2017causal,madras2018fairness} use deep learning latent variable models: \cite{kusner2018counterfactual,louizos2017causal} consider latent confounders and  \cite {madras2018fairness} considers the sensitive attribute as a  confounder. The situations in which these assumptions apply are quite different from the problem we study, since they aim at eliminating the effect of the confounder, while for us the unobservable variable is a mediator, and we want to use it as the basis for a fair decision.  As a consequence, the notion of fairness those works aim at achieving is not suitable for our case. 
% \cite{kusner2018counterfactual} also considers latent confounders between the sensitive attribute and the decision, and aims at achieving a counterfactual fairness notion. 
\cite{chiappa2019path} introduces path-specific counterfactual fairness, where (among other cases) they consider the latent cause of a mediator between the sensitive attribute and the decision. This is more similar to our notion of fairness. However, \cite{chiappa2019path} assumes that the latent variable is independent from the sensitive attribute; as such, their method is not directly applicable to our problem. 
\cite{choi2021group} uses probabilistic circuits to impose statistical parity and to learn a relationship between the latent fair decision and other variables. 
Finally, \cite{Feldman} uses a notion of fairness called \emph{disparate impact}, which is similar to statistical disparity, except that it is defined as a ratio (instead of a difference) between the probabilities of positive decisions for each group. Similarly to our work, \cite{Feldman} applies a corrective factor to the outcome of the observed variable $Z$, but their goal is to minimize the disparate impact (within a certain allowed threshold $\alpha$), which is again in the spirit of minimizing statistical disparity. Also their  technique is very different: they consider the  distributions on the observed variable $Z$ for each group, and they compute new distributions that minimize  the earth movers' distance and achieve the threshold $\alpha$. Then, they map each value of $Z$ (for each group) on the new distribution so to maintain the percentile. 

%  {\color{blue}{For Ruta: Please check whether there are other papers that work with latent variables. Check also if any of the cited papers has been published, and update the bib with the published version}}

% {\color{purple}{Rewritten up to here}}

% Other approaches aiming to satisfy statistical parity or Disparate Impact fairness notions are \cite{kamiran2012data,calmon2017optimized,wei2020optimized,sharma2020data} using optimization, sampling and re-weighting techniques, \cite{zemel2013learning,louizos2015variational} applying deep learning or based on optimal transport methods \cite{gordaliza2019obtaining}. Less numerous methods aiming to satisfy individual fairness are mostly based on K-Nearest Neighbour clustering \cite{luong2011k}.
% A distinct group of pre-processing approaches use GANs (Generative Neural Networks) to generate fair data \cite{rajabi2022tabfairgan,sattigeri2019fairness,zhang2018mitigating,hong2021federated, reimers2021towards, zhang2020towards} including the instance based on causal graphs \cite{xu2019achieving}. 
% Finally a line of pre-processing methods satisfying counterfactual fairness include \cite{kusner2018causal,zmigrod2019counterfactual,joo2020gender,pitis2020counterfactual,de2021transport}.

\section{Preliminaries and Notation}\label{sec:preliminaries}
%We recall some basic concepts that will be used throughout the paper. The reader who is already familiar with one or more of them can skip the corresponding section(s).

\paragraph{$\hat{E}$, $\hat{Y}$ and $Y$ notations} In this paper, $\hat{E}$ (with generic value $\hat{e}$) represents the estimation of the explanatory variable $E$. Similarly, $\hat{Y}_{\hat{E}}$ (with generic value $\hat{y}$) represents the estimation of the decision, based on  $\hat{E}$, rather than the prediction of the  model. To put it in context,  recall that we are proposing a pre-processing method: $\hat{y}$ represents the value that we assign as decision in a sample of the training data during the pre-processing phase. The fairness and precision notions are defined with respect to these estimations.
We   use $Y_Z$  to indicate the biased decision based on $Z$, and $Y_E$ for the ``true'' decision based on $E$. When clear from the context, we may use $Y$ instead of $Y_E$.

\paragraph{The Expectation-Maximization Framework}
\label{sec:em}
\label{sec:ml}  
%A statistical model is often used to explain the observed output  of a system. 
Let  $O$ be a random variable depending on an unknown parameter $\vtheta$.
Given that we observe $O = o$, the aim is to find the value of $\vtheta$ that maximizes the
probability of this observation, and that therefore is \emph{its best explanation}.   
To this purpose, we use the \emph{log-likelihood function}  
$L(\vtheta) = \log \pr[O=o | \vtheta]$.
%
%where $\pr[O=o | \vtheta]$ is the conditional probability of $O=o$ given $\vtheta$. 
A Maximum-Likelihood Estimation (MLE) of the parameter is then defined as $\argmax_\theta L(\vtheta)$ (which is  the $\theta$ that maximizes $\pr[O=o | \vtheta]$, since $\log$ is monotone). 
The Expectation-Maximization (EM) framework \cite{Dempster:77:RSS,mclachlan2007algorithm,Wu:83:jastat} is a powerful 
method for computing $\argmax_\theta L(\vtheta)$.

\subsection{Metrics for the quality of estimations}\label{sec:metric for quality}

\paragraph{The Wasserstein distance}
This  distance is  defined between probability distributions on a  metric space. 
%Since it takes into account the ground distance, it is more refined  than Total Variation, and it is more suitable for measuring the precision of a probability estimation when the problem at hand involves decisions based on the ground distance. 
%
Let $\mathcal X$ be a set provided with a distance $d$, and $\mu,\nu$ be two discrete probability distributions on $\mathcal X$. The Wasserstein distance between $\mu$ and $\nu$ is defined as
\begin{equation}\label{eq:Wass}
    {\mathcal W} (\mu,\nu) \;=\; \min_\alpha \sum_{x,y\in {\mathcal X}} \alpha(x,y)\,d(x,y),
\end{equation}
where $\alpha$ represents a \emph{coupling}, i.e., a joint distributions with marginals $\mu$ and $\nu$ satisfying the properties
$ \sum_{y\in {\mathcal X}} \alpha(x,y) \;=\; \mu(x)$ and 
$\sum_{x\in {\mathcal X}} \alpha(x,y) \;=\; \nu(y)$.

% Intuitively, it is defined as the minimum ``cost'' necessary for making the two distributions equal, where the cost is defined in terms of the probability mass to be transported from one distribution to the other, and the ground distance between the origin and the target of the transportation. 

% \subsection{Causal Graph}
% Following %the J.Pearl framework 
% \cite{pearl_causality_2009}, a causal graph $G = (V, E)$ is a directed acyclic graph (DAG) composed of a set of variables/vertices $V$ and a set of edges $E$ that describe the dependence relations and causal directions between the variables.

\paragraph{Accuracy}\label{sec:accuracy}
Let $X,Y$ be two random variables with support ${\mathcal X}$ and ${\mathcal Y}$ respectively, and joint distribution $\pr[X,Y]$. Let 
 $f:{\mathcal X}\rightarrow {\mathcal Y}$ be a function that, given $x\in {\mathcal X}$,  \emph{estimates} the corresponding $y$, and let $\hat{y}$ be the result, i.e., $\hat{y}=f(x)$. The accuracy of $f$ is defined as the expected value of $\mathds{1}_{\hat{y}=y}$, that is the function that gives $1$ if $\hat{y}=y$, and $0$ otherwise.
When the distribution is unknown,  the accuracy is estimated empirically via a set of pairs 
 $\{(x_i,y_i)\mid i\in {\mathcal I}\}$ independently sampled from $\pr[X,Y]$ (\emph{testing set}), and is defined as  
\vspace{-1mm}
\begin{equation}\label{eq:Acc}
    \mathit{Acc}(\hat{Y}, Y) \;= \; \frac{1}{|{\mathcal I}|}\;\sum_{i\in{\mathcal I}}\mathds{1}_{\hat{y}_i=y_i}
\;\mbox{where}\;\hat{y}_i = f(x_i).
\end{equation}

 % We will use the accuracy to evaluate the prediction of $Y$, and also of $E$. 
 % In the first case, the testing set is of the form 
 % $\{((s_i,z_i),y_i)\in {\mathcal S}\times{\mathcal Z}\times {\mathcal Y}\mid i\in {\mathcal I}\}$. In the second case, the testing set is of the form 
 % $\{((s_i,z_i),e_i)\in {\mathcal S}\times{\mathcal Z}\times {\mathcal E}\mid i\in {\mathcal I}\}$.

\paragraph{Distortion}\label{sec:distortion}
If the variable to be predicted ranges over a metric space, and the metric is important for decision-making (like the case of $E$ in our  examples), accuracy is not always the best way to measure the quality of the estimation. Arguably, it is more suitable 
to use the \emph{distortion}, i.e.,  the expected distance between the true value and its estimation. Using the testing set $\{((z_i,s_i),e_i)\mid i\in {\mathcal I}\}$, the  distortion in the estimation of $E$ is  defined as 
\vspace{-2mm}

\begin{equation}\label{eq:Dist}
    \mathit{Dist}(\hat{E}, E) =  \frac{1}{|{\mathcal I}|}\;\sum_{i\in{\mathcal I}} |\hat{e}_i-e_i|, 
 \mbox{ where }\hat{e}_i = f(z_i,s_i).
\end{equation}

\subsection{\bf Metrics for fairness}\label{sec:fairness metrics}
SP, CSP, and EO 
are rarely achieved, since  they require a perfect match. It is therefore useful to quantify the level of (un)fairness, i.e., the difference between the two groups. We will use the following metrics: 

\begin{description}
 \item[Statistical parity difference]   (SPD) 
\begin{equation}
\label{eq:SPD}
\pr[\hat{Y}_{\hat{E}}=1|S=1] - \pr[\hat{Y}_{\hat{E}}=1|S=0].
\end{equation}
 
\item[Conditional statistical parity difference]  ({CSPD})  
\begin{equation}
\label{eq:CSPD}
\pr[\hat{Y}_{\hat{E}}=1|E,S=1]- \pr[\hat{Y}_{\hat{E}}=1|E,S=0].
\end{equation}
% \item[Expected conditional stat. parity difference] ({CSPD})
% \begin{equation}
% \label{eq:ECSPD}
% \sum_e \pr[E=e] \;\mbox{{CSPD}}_e. 
% \end{equation}
\item[Equal opportunity difference]  
     (EOD)
     \begin{equation}
\label{eq:EOD}
\pr[\hat{Y}_{\hat{E}}=1|Y_E=1,S=1] - \pr[\hat{Y}_{\hat{E}}=1|Y_E=1,S=0].
     \end{equation}
\end{description}

\section{The BaBE method}

In this section we describe the BaBE approach.
We briefly recall the problem: we have a data model represented in Figure \ref{fig:mainmodels}, where $S$ is the sensitive attribute, $E$ is the explanatory variable on which a fair decision should be based,  and  $Z$ is an observed but biased version of $E$. We need to estimate the distribution $\pr[E|Z,S]$. The first step is to estimate the distribution of $E$ for each group, $\pr[E|S]$. We accomplish this task by adapting the Expectation-Maximization method to our particular setting. 
Then, from $\pr[E|S]$ we  derive, using the Bayes theorem, the estimation of $\hat \pr[E|S,Z]$, from which we finally derive $\hat{E}$ and $\hat{Y}_{\hat{E}}$. The pipeline of the process is provided in Figure~\ref{fig:pipeline}.

\begin{figure}
    \centering
    \includegraphics[height=3cm]{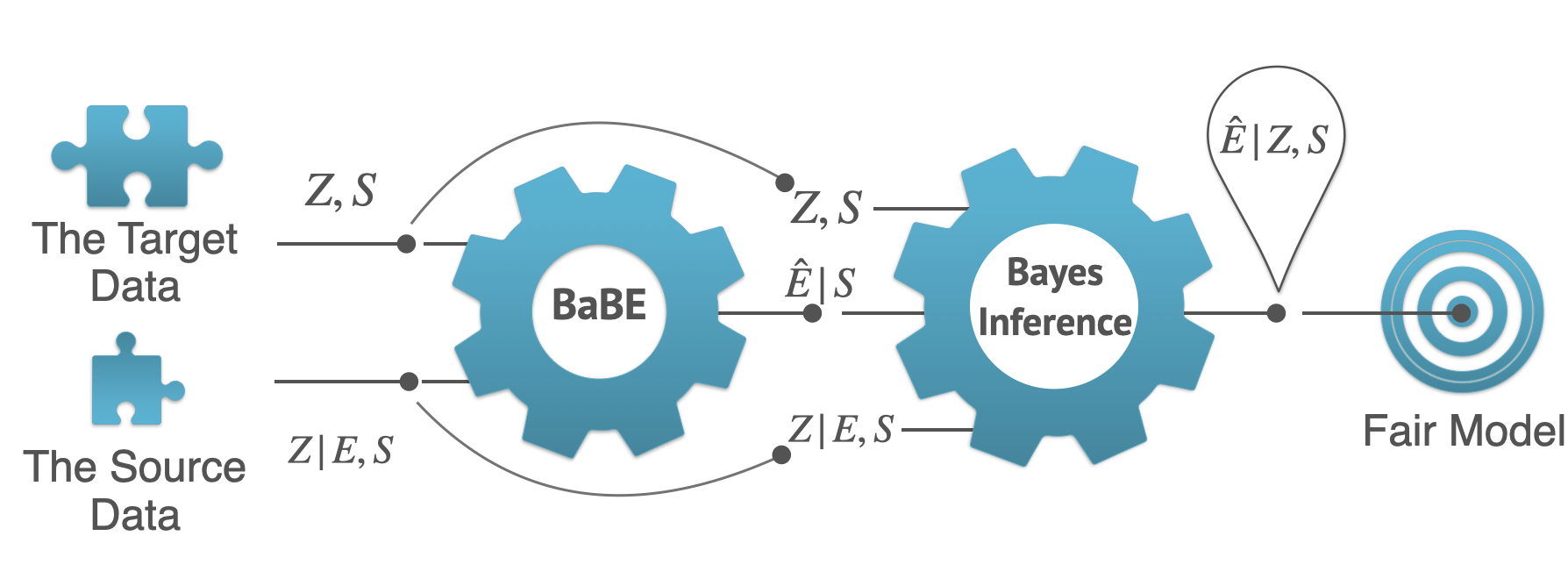}
    \caption{The pipeline of BaBE application. The variable $E$ is observable in the source data and $\pr[Z|E,S]$ can be derived. The target data is the one where $E$ is not observable and we want to recover it using  $\pr[Z|E,S]$ derived from the source data. We input $\pr[Z|E,S]$ and statistics from the observable variables in the target data to BaBE and estimate $\hat{E}$ consistent with the target distribution (possibly different than in the source data). We then again use $\pr[Z|E,S]$ (from source data), observable variables (from the target data) and $\hat{E}$ (BaBE estimation) to inference $\hat{E}|Z,S$ for each sample in the target data. }
    \label{fig:pipeline}
\end{figure}

\subsection{Deriving $\hat \pr[E|S]$}
\label{sec:algibu}

% The BaBE algorithm is a particular case of the Expectation Maximization algorithms. It is used in the Privacy community to recover aggregate data distributions from the added noise~\cite{agrawal_ibu, EhabFullConv}. 
% We assume a random variable $E|S$ on sets $\mathcal E$ $\mathcal S$ with distribution $P_{E|S}$; we denote with $P_{E|S}[e|s]$ the probability that $E=e$ given $S=s$ (for $e\in \mathcal{E}$ and $s \in {\mathcal S}$). 
% We also assume $N$ i.i.d. (independent and identically distributed) samples of $E|S$, call them $E_1|S_1$, $E_2|S_2$, . . . , $E_N|S_N$, each following distribution $P_{E|S}$. 
% We assume that the value of every $E_i$ is independently biased by a discriminatory mechanism dependent on or equivalent to the sensitive attribute $P_{Z|E,S}[z|e,s]$ to produce a biased value $z_i \in {\mathcal Z}$ that we observe in the biased data set. This yields a new random variable $Z$ on $\mathcal Z$ and let us denote with $\bar z$ the sequence of biased values $z_1,\ldots,z_N$.

We estimate  the unknown parameter $\pr[E|S]$ as the MLE of a sequence of samples $(\bar{z},\bar{s}) \ =\{(z_i,s_i)\,|\,i \in [1,N]\}$,\footnote{We use the notation $[a,b]$ to represent the integers from $a$ to $b$.}  assuming that we know the effect of the bias, i.e., $\pr[Z|E,S]$. 
%We refer to Appendix~\ref{appendix} for the technical details. 
% Given the observations $z_i$ for $i = 1, 2, . . . , N$, the BaBE approximates a Maximum Likelihood Estimate (MLE for short) for
% them as follows.
%
We denote by $\varphi_s[z,\bar z]$ the empirical probability of $Z=z$ given $S=s$, i.e., the frequency of $z$ in the samples with $S=s$.  Algorithm~\ref{alg:IBU} estimates $\pr[E|S]$ by starting with the uniform distribution and by iteratively computing at step $t$ a new estimation $\hat \pr[E|S]^{(t)}$ from the previous one,  getting closer and closer to the MLE. The proof of correctness of Algorithm~\ref{alg:IBU} is provided in   \cite{binkyte2023babe} (the archival version of this paper). 

We have experimentally verified that our method is quite efficient: 
The running time of Algorithm~\ref{alg:IBU} on the data of Section~\ref{sec:experiments} is a few seconds. Details are reported in the additional material.  

% \begin{algorithm}[tb]
% \caption{\;Estimation of $\pr[E|S]$}
% \label{alg:IBU}
% \textbf{Input}:\small $\{(z_i,s_i)\,|\,i\in[1,N]\}$ , $\pr[Z=z|E=e,S=s]\}$ and $\gamma$ (desired level of precision) \\
% % \textbf{Parameter}: Optional list of parameters\\
% \textbf{Output}: 
% An approximation $\hat{P}[E|S]$ of the MLE\\[-2ex]
% \begin{algorithmic}[1] %[1] enables line numbers
% \STATE  Compute $\varphi_s[z, \bar z]$  for all $z, s$.
% \STATE   $\hat \pr[E=e|S=s]^{(0)} = \frac 1 {|{E}|}$  for all $e$, where $|{E}|$ is the cardinality of  $E$.
% \STATE $t = 0$
% \REPEAT
% \STATE $t=t+1$
% \STATE $\hat{P}[E=e|S=s]^{(t)}$ = \\
% \qquad $\sum\limits_{z \in {\mathcal Z}}\varphi_s[z, \bar z] \frac{\pr[Z=z|E=e,S=s]\hat{P}[E=e|S=s]^{(t-1)}}{\sum\limits_{e' \in {\mathcal E}}\pr[Z=z|E=e',S=s]\hat{P}[E=e'|S=s]^{(t-1)}}$, \\
% for all $e , s$
% \UNTIL {$\left|\hat{P}[E=e|S=s]^{(t)}-\hat{P}[E=e|S=s]^{(t-1)} \right| < \gamma $, for all $e , s$}
% \STATE \textbf{return} $\hat{P}[E|S] = \hat{P}[E|S]^{(t)}$
% \end{algorithmic}
% \end{algorithm}

\begin{algorithm}[ht]
  \caption{BaBE: Bayesian Bias Elimination}
  \label{alg:IBU}
  \begin{algorithmic}
    %\SetAlgoLined
   \State \textbf{Data}: $\{(z_i$ , $\pr[Z=z_i|E=e,S=s])\}_{i \in \{1..N\}}$
   %, where $z_i$ is the $i^{th}$ observable and $\pr[Z=z|E=e,S=s]$ is the conditional probability table defining the mechanism used to yield $z$ from $e$.
   and $\gamma$ (an allowed error in estimating $\pr[E=e|S=s]$)
   
   \State \textbf{Result}: An approximation (up to $\gamma$) $\hat \pr[E|S]$ of the MLE\\

  \State Compute $\varphi_s[z, \bar z]$, for every $z \in \mathcal Z$ and $s\in \mathcal S$
  
  \State $\hat \pr[E=e|S=s]^{(0)} = \frac 1 {|{\mathcal E}|}$, for every $e\in \mathcal E$
  \State $t = 0$
  
  \Repeat
   \State $t=t+1$
 \State $\hat \pr[E=e|S=s]^{(t)} =\sum\limits_{z \in {\mathcal Z}}\varphi_s[z, \bar z] \frac{\pr[Z=z|E=e,S=s]\hat \pr[E=e|S=s]^{(t-1)}}{\sum\limits_{e' \in {\mathcal E}}\pr[Z=z|E=e',S=s]\hat \pr[E=e'|S=s]^{(t-1)}}$, for every $e \in {\mathcal E}$ and $s\in \mathcal S$
 \Until $\ \forall e \in {\mathcal E}\, \forall s\in \mathcal S.\, \left|\hat \pr[E=e|S=s]^{(t)}-\hat \pr[E=e|S=s]^{(t-1)} \right| < \gamma $
 
 \State \Return $\hat \pr[E|S] = \hat \pr[E|S]^{(t)}$
% \label{alg1-line16}
\end{algorithmic}
\end{algorithm}

% We remark that a similar adaptation of the EM algorithm, for a totally different problem, appears in the privacy literature \cite{agrawal_ibu,elsalamouny2019convergence} under the name of \emph{Iterative Bayesian Update}. In that case, the 
% variable to estimate was the original distribution of data obfuscated by a privacy-protection mechanism. The problem in that setting was somewhat simpler, because it involved only two variables (the original distribution and the empirical one on the observed obfuscated data) and there was no conditioning. 

\subsection{Deriving $\hat \pr[E|Z,S]$  from $\hat \pr[E|S]$}
\label{sec:decode}

Given the data $\{(z_i,s_i)\,|\,i\in[1,N]\}$,
the conditional distributions $\pr[Z|E,S]$, and  the estimation $\hat \pr[E|S]$, 
we use the Bayes formula to
estimate $\hat \pr[E=e|Z=z,S=s]$ as
\begin{align*}
\textstyle
\hat \pr[E=e|Z=z,S=s] = \frac{\pr[Z=z|E=e,S=s]\hat \pr[E=e|S=s]}{\pr[Z=z|S=s]}
\label{eq:inf_edz}
\end{align*}

\subsection{Deriving ${\hat{E}}$ and $\hat{Y}_{\hat{E}}$ from $\hat\pr [E|Z,S]$}\label{sec:Y from E}

% For real-world classifications, we usually need not only the estimation of $E$, but a prediction $\hat Y$ that is based on that estimation. This is the function $f$ illustrated in Section~\ref{sec:accuracy}.
% We also need the decision to measure SP and CSP of the pre-processed data. 
We propose two ways to derive $\hat{Y}_{\hat{E}}$ for pre-processing the samples in the training data, depending on how much probability mass is concentrated on the mode of $\hat \pr[E|Z,S]$. We denote by $\tau$ the threshold for the values of $E$ that qualify for the positive decision. 
%%%%%. 

\paragraph{Method 1}\label{method-1} Given $z$ and $s$,   if $\hat \pr[E|Z=z,S=s]$ is unimodal and has a large probability mass (say, 50\% or more)  on its mode, 
then we can safely set $\hat E$ to be that mode. 
Namely, if 
$\max_e  \hat \pr[E=e|Z=z,S=s]\geq 0.5$ then we  set $\hat e = \argmax_e  \hat \pr[E=e|Z=z,S=s]$, and we 
can then use $\hat e$ directly to set $\hat{Y}_{\hat{E}}=1$ or $\hat{Y}_{\hat{E}}=0$ in those samples with  $Z=z$ and $S=s$, depending on whether $\hat e \geq \tau$ or not, respectively.
Our experimental results show that this method gives a good accuracy.

\paragraph{Method 2}\label{method-2}  
If $\hat \pr[E|Z=z,S=s]$ is dispersed on several values, so that no value is strongly predominant, then it is  impossible to estimate individual values for $E$ with high accuracy. However, we can still accurately estimate $Y_E$  as follows:  
Let $\sigma_0 = \sum_{e<\tau} \hat \pr[E=e|Z=z,S=s]$  and $\sigma_1= \sum_{e\geq\tau} \hat \pr[E=e|Z=z,S=s]$. If $\sigma_0 < \sigma_1$, then we set  $\hat{Y}_{\hat{E}}=1$; otherwise, $\hat{Y}_{\hat{E}}=0$. 
% Formally: 
% \begin{align*}
% \textstyle
%     &\mbox{if}\quad 
%     \sum_{e\geq \tau} \hat \pr[E=e|Z=z,S=s] \geq 0.5 \\[-1ex]
%     &\quad \mbox{then set}\quad \hat Y|Z=z,S=s =1, \\[-1ex]
%     &\quad\mbox{else set}\quad \hat Y|Z=z,S=s =0. 
% \end{align*}

% We will give an example of the application of this method in Section~\ref{sec:experiments}. 

\section{Experiments}\label{sec:experiments}

In this section, we test BaBE on scenarios corresponding to Examples~\ref{exa:bias} and \ref{exa:bias2},  using  synthetic data sets and a real data set respectively. 
%(the description of the data sets is in Section~\ref{sub:data}). 
We compare our results with those achieved by the following well-known pre-processing approaches that aim to satisfy statistical parity, as well as machine learning algorithms trained on the data set where $E$ is observable.

\subsection{Metrics}
 We will use the following metrics to measure fairness: Statistical parity difference ($SPD$, Equation~\ref{eq:SPD}), Conditional statistical parity difference ($CSPD$, Equation~\ref{eq:CSPD}), Equal opportunity difference ($EOD$, Equation~\ref{eq:EOD}).
 The performance is measured by accuracy ($\mathit{Acc}(\hat{Y}, Y)$, Equation~\ref{eq:Acc}), distortion ($\mathit{Dist}(\hat{E}, E)$, Equation~\ref{eq:Dist}), and the Wasserstein distance between the true and estimated distributions ($\mathcal W (\mu,\nu)$, Equation~\ref{eq:Wass}).

\subsection{Other Algorithms for Comparison}

The first approach we compare with is the \emph{disparate impact} (DI)  \emph{remover} ~\cite{bellamy2019ai,Feldman}.\footnote{We use the implementation by \cite{bellamy2019ai}.}
% This algorithm works on selected features to equalize their distributions for the groups in $S$. 
DI has a parameter $\lambda$, which represents the minimum allowed ratio between the probability of success ($\hat{Y}=1$) of each group (hence $\lambda=1$ corresponds to statistical parity). For the experiments, we use $\lambda=0.8$. 

The second algorithm we compare with ours is the \emph{naive Bayes} (NB)~\cite{calders_nbayes}.\footnote{Implementation kindly provided by the authors of \cite{calders_nbayes}.} NB also 
% considers prior knowledge and 
applies the EM method; however, in contrast to our work, NB assumes that $E$ and  $S$ are independent, and uses EM to take decisions  that optimize the trade-off between SPD and accuracy. 
% Their method pre-processes binary labels rather than multi-valued features, therefore  we can only compare it to our approach after obtaining the decisions $\hat{Y}$ based on the selected threshold in the pre-processed variable $\hat{E}$.

Finally, we compare the performance of BaBE with the ML methods that are trained on the data where $E$ is observed (the source data). The model is then used to predict $E$ (from $S$ and $Z$) in the data sets where the distribution of $E|S$ is different from the source data, but the learned mechanism ($Z|E,S$) is the same. We used linear regression (LG) and decision tree regression (DT) for the experiments.\footnote{We use scikit-learn implementations of the machine learning algorithms~\cite{pedregosa2011scikit}.  }

%  We measure the performance of the prediction $\hat{Y}$ of  $Y$ with the accuracy~\footnote{We use the ScikitLearn implementation~\cite{pedregosa2011scikit}.}, defined by the formula 
%  \begin{equation}
%      {\mathit Acc} \ = \  \frac{TP+TN}{TP+TN+FP+FN}       
%  \end{equation}
% where $TP$, $TN$,$FP$ and $FN$ represent the True Positives, True Negatives, False Positives and False Negatives, respectively. 
%  To evaluate the estimation of $E|S$ we use both the accuracy and the Wasserstein distance . The latter gives us better judgement of the performance, because the proximity to the real value is also taken into account. For example, the accuracy would show the same results if for the real value 2, the predicted value is 4 or 50. Instead, the Wasserstein distance would favor the closer numerical value.
% To evaluate the fairness of the data, we use conditional statistical disparity  with respect to the true $E$. We also measure statistical disparity  and accuracy difference between the groups: $Acc_{Difference}=Acc_{S=0}-Acc_{S=1}$.

% Definition for the Manhattan Distance:

% $$\mathbb{E}_E d(E, \hat{E}) = \sum_e p(e) |p(e)-\hat{\pr}(e) |$$

% \subsubsection{Data}
% \label{sub:data}

\subsection{Synthetic data sets with distributions shifts}

This group of experiments is aimed at testing how BaBE copes with the transfer of knowledge to populations with different distributions. For this purpose, we generate a synthetic set, that we call "source data", 
where the mean of $E$ for group $0$, $\mathit{mean0}$, is $40$ and the mean of $E$ for group $1$, $\mathit{mean1}$, is $80$. The groups in this data set are about even in size. We use this set of ``source data'' to estimate the distributions $\pr[Z|E,S]$. 

Then, we generate three different 
data sets $1$, $2$ and $3$ where $\mathit{mean1}$ is still $80$, while $\mathit{mean0}$ varies from $40$ to $80$, representing a distribution shift, w.r.t. the source data, on the $E$ for group $0$. Varying $\mathit{mean0}$ will also allow us to validate the claim that our method works well regardless of  $E$ being independent of $S$ or not. 
The percentage of the two groups in these new data sets also changes: we have set the group $1$ to be $60$\% of the population, and, consequently, the group $0$ to be $40$\%. 

\subsubsection{Generation of the synthetic data sets}
In this section we explain how to generate various data sets containing tuples of the form $(s_i, e_i, z_i, y_i)$.
First, we generate a data set of $30$K  elements 
$\{s_i\}_{i\in [1,30{\text K}]}$ representing  values for the sensitive variable (group) $S$, where each $s_i$ is sampled
from the Bernoulli distribution $\mathcal{B}(0.5)$. This means that the two groups are about even. Then, we set the domain of $E$ to be
equal to $[0,99]$, and to each of the elements $s_i$ in the sequence we associate a value $e_i$  for the variable $E$, sampled from the normal distribution $\mathcal{N}(\mathit{mean0}, sd)$, if $s_i=0$,  and from $\mathcal{N}(\mathit{mean1}, sd)$, if $s_i=1$,\footnote{To keep the samples in the range of $E$, we re-sample the values that are lower than $0$ or higher than $99$. We also discretize them by rounding to the nearest integer.} 
where the mean $\mathit{mean1}$ is set to $80$, and the standard deviation $sd$ is set to $30$. 
On the other hand, the value of $\mathit{mean0}$, is $40$ in the source data, and varies from $40$ to $80$ in the data sets $1$, $2$ and $3$.
Finally, to each pair $(s_i,e_i)$ we associate a value $z_i$ with $Z$ by applying a bias to $e_i$. More precisely, 
 $z_i = 100\times sigmoid(e_i/10-5) - 100\times sigmoid(e_i/10-5)\times 0.2$ + $\varepsilon$ for $S=0$ and $z_i = 100\times sigmoid(e_i/10-5) + 100\times sigmoid(e_i/10-5)\times 0.02$ + $\varepsilon$ for $S=1$, where $\varepsilon$  is a noise term sampled from $\mathcal{N}(1, 0.05)$. The threshold for the decision is $E = 80$, namely: $Y=1$ if $E> 80$ and $Y=0$ otherwise. This is used to associate a decision $y_i$ to each tuple $(s_i,e_i,z_i)$. The distribution of $E|S$  for the source data and the data sets $1$-$3$ are shown in Figure~\ref{fig:Dist2-source}.

 \begin{figure}[t]
    \centering
    \includegraphics[width=\linewidth]{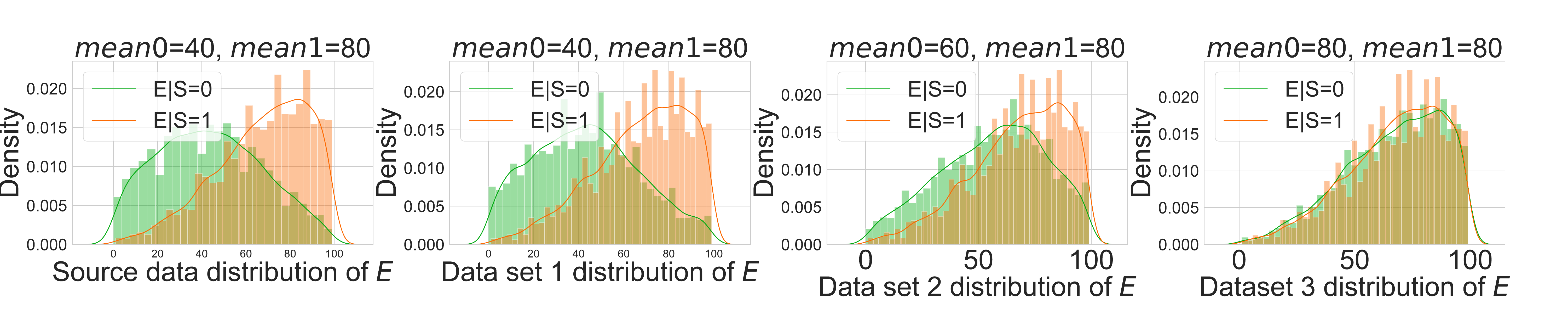}
     \vspace{-5mm}
     \caption{The distribution of $E|S$ in the source data and  in the new populations. }
    \label{fig:Dist2-source}
\end{figure}

\subsubsection{Application of BaBE}
We use the source data to estimate $\pr[Z|E,S]$. This conditional probability is then used  to estimate $\pr[E|S]$ in the other data sets, where it is different from the "source data", in the following way:
We take a random subset of the data from the data sets $1$-$3$ ($80\%$), remove the $E$ values from them, and use them to compute the empirical distribution $\pr[Z|S]$ and to produce the estimate
$\hat \pr[E|S]$ by applying our BaBE method.

We verify that these data sets satisfy the conditions for  Method 1 (cf. Section ~\ref{method-1}),
%of Section~\ref{sec:Y from E}
and we apply this method to the remaining ($20\%$) of the data (testing data sets) to infer the values of $\hat{E}$ and $\hat{Y}_{\hat{E}}$ for each sample. We compare our estimates to the true values of $E$ and $Y$ in the testing data. 
% For example, if for the value of $Z=z$ and $S=s$ the highest probability is to observe $E=e$, we change the existing $z$ to $e$ for all the records with values $z$ and $s$. 

% For Figures 2,3,4 the calculation is done by computing the distances separately for each group in $S$, and then averaging. This is done to avoid the compensation effect, since the bias is negative for group $0$ and positive for group $1$. 

% \begin{figure}[H]
%     \centering
%     \includegraphics[width=\linewidth]{Graphs/Experiment3/e_given_s2.pdf}
%     \caption{The distribution of $E|S$ in data sets 2 and 3.}
%     \label{fig:my_label1}
% \end{figure}

% \begin{figure}[H]
%     \centering
%     \includegraphics[width=\linewidth]{Graphs/Experiment3/e_given_s3.pdf}
%     \caption{The distribution of $E|S$ in data sets 4 and 5.}
%     \label{fig:my_label1}
% \end{figure}

% \subsubsection{Results on synthetic data sets with distribution shift}

Figure~\ref{fig:Wass2}  shows the Wasserstein distances between the true distributions and the estimated ones.  As we can see, BaBE manages to estimate $E$ quite well: the distance w.r.t. $E$ is very small.                                                    
                                        
% \begin{figure}[H]
%     \centering
%     \includegraphics[width=\linewidth]{Babe/Graphs/Experiment3/3_w1.pdf}
%     \caption{The Wasserstein distance between $\hat{\pr}[Z|S=1]$ and $\pr[E|S=1]$
%     and between $\hat{\pr}[E|S=1]$ and $\pr[E|S=1]$.}
%     \label{fig:Wass2-1}
% \end{figure}

% \begin{figure}[H]
%     \centering
%     \includegraphics[width=\linewidth]{Babe/Graphs/Experiment3/3_w0.pdf}
%     \caption{The Wasserstein distance between $\hat{\pr}[Z|S=0]$ and $\pr[E|S=0]$
%     and between $\hat{\pr}[E|S=0]$ and $\pr[E|S=0]$.}
%     \label{fig:Wass2-0}

% \end{figure}

\begin{figure}[h]
    \centering
    \includegraphics[width=0.8\linewidth]{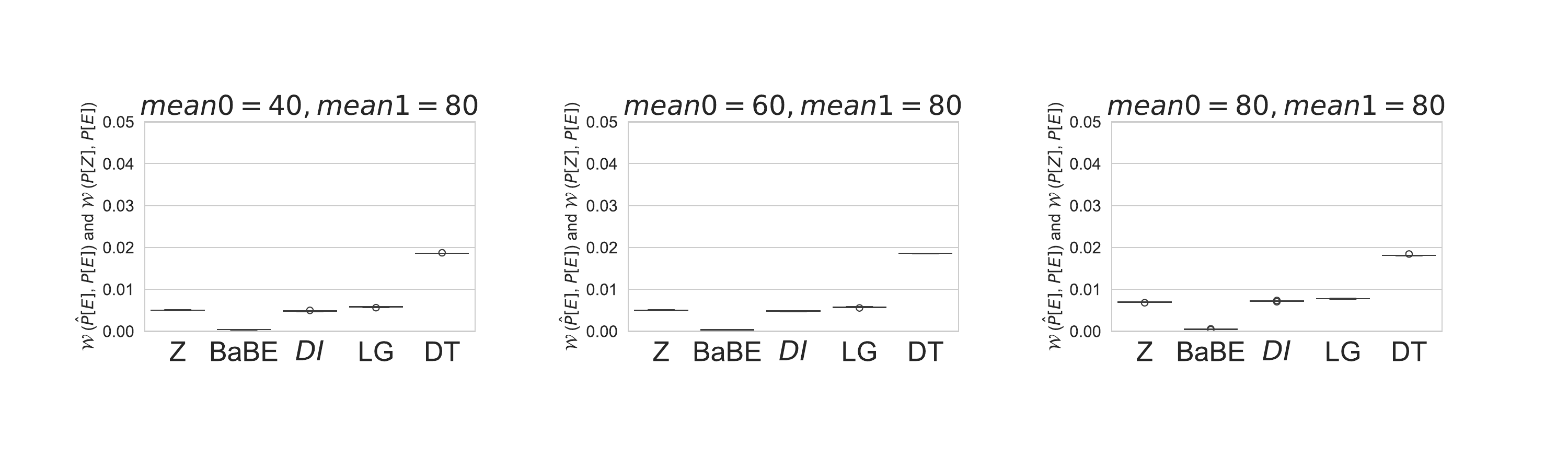}
     \vspace{-5mm}
     \caption{Experiments on the synthetic data sets: The Wasserstein distance between $\hat{\pr}[Z]$ and $\pr[E]$
    and between $\hat{\pr}[E]$ and $\pr[E]$. }
    \label{fig:Wass2}
\end{figure}

% \begin{figure}[t]
%     \centering
%     \includegraphics[width=\linewidth]{Graphs/Experiment3/3_dist_e1.pdf}
%     \caption{ The distortion between ${Z}|S=1$ and $E|S=1$ (for $Z$), and between $\hat{E}|S=1$ and ${E}|S=1$ (for BaBE and DI)}

% \end{figure}

% \begin{figure}[t]
%     \centering
%     \includegraphics[width=\linewidth]{Graphs/Experiment3/3_dist_e0.pdf}
%     \caption{ The distortion between ${Z}|S=0$ and $E|S=0$ (for $Z$), and between $\hat{E}|S=0$ and ${E}|S=0$ (for BaBE and DI)}
Figure~\ref{fig:ex3_accE} shows the accuracy with respect to the true $E$ (discretized values). BaBE is able to achieve a much better accuracy than other methods. DT algorithm is the second best performer; however DT is still performing worse than BaBE, despite being trained on the data set where $E$ is observable.
% \end{figure}
\begin{figure}[h]
    \centering
    \includegraphics[width=0.8\linewidth]{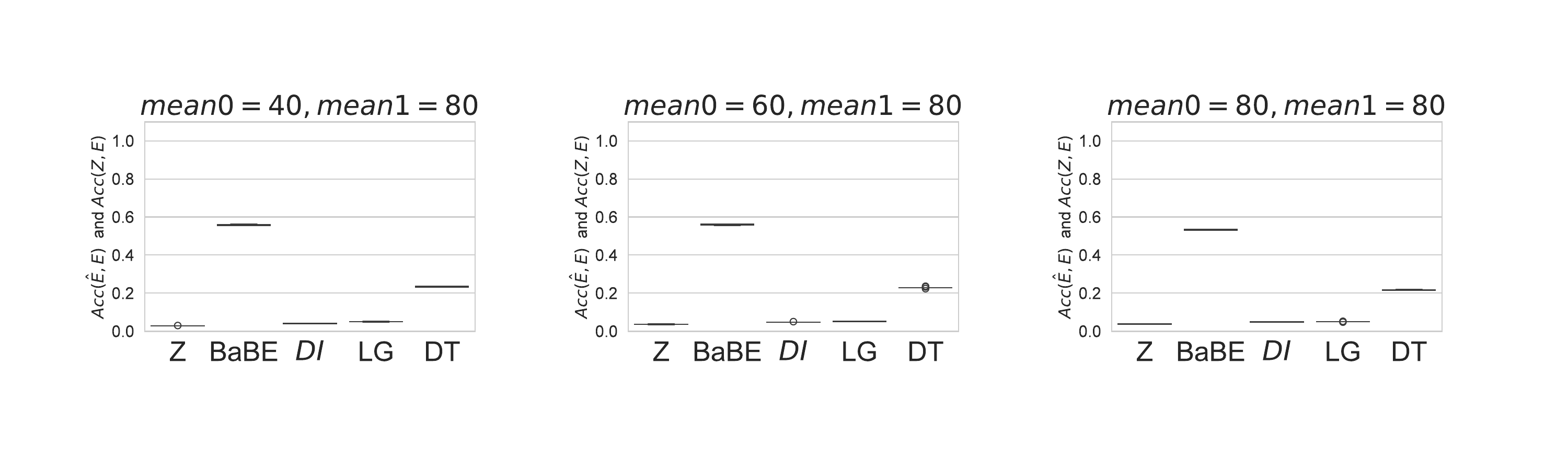}
     \vspace{-5mm}
     \caption{Experiments on the synthetic data sets: The accuracy between $Z$ and $E$,  and between $\hat{E}$ and $E$. }
\label{fig:ex3_accE}
\end{figure}

% \begin{figure}[t]
%     \centering
%     \includegraphics[width=\linewidth]{Graphs/Experiment3/3_dist_e.pdf}
%     \caption{The distortion between $Z$ and $E$ (for $Z$), and between $\hat{E}$ and ${E}$ (for BaBE and DI)}
%     % \label{fig:my_label1}
% \end{figure}

% \begin{figure}[H]
%     \centering
%     \includegraphics[width=\linewidth]{Babe/Graphs/Experiment3/3_acc_y1.pdf}
%     \caption{Experiments on the transfer of knowledge: The accuracy of $\hat{Y}_{Z}|S=1$ and $\hat{Y}_{\hat{E}}|S=1$ w.r.t. $Y_E|S=1$.  }
%     \label{fig:Acc2-1}
% \end{figure}

% \begin{figure}[H]
%     \centering
%     \includegraphics[width=\linewidth]{Babe/Graphs/Experiment3/3_acc_y0.pdf}
%     \caption{Experiment on the transfer of knowledge: The accuracy of $\hat{Y}_{Z}|S=0$ and $\hat{Y}_{\hat{E}}|S=0$ w.r.t. $Y_E|S=0$. }
%  \label{fig:Acc2-0}
% \end{figure}

Figure~\ref{fig:Acc2} shows the accuracy with respect to $Y$ for the two groups. Once again the performance of BaBE is better than other pre-processing methods. The overall performance of all methods is better than measuring the accuracy with respect to $E$. This is not surprising, as achieving good accuracy in a binary setting is an easier task. DT achieves almost the same accuracy as BaBE on the dataset 1 ($mean0=40, mean1=80$) which has the same distribution of $E|S$ as the training data. However, the performance of DT decreases on the data sets where the distribution of $E|S$ is different from the training data ($mean0=60$ and $mean0=80$).

\begin{figure}[h]
    \centering
    \includegraphics[width=0.8\linewidth]{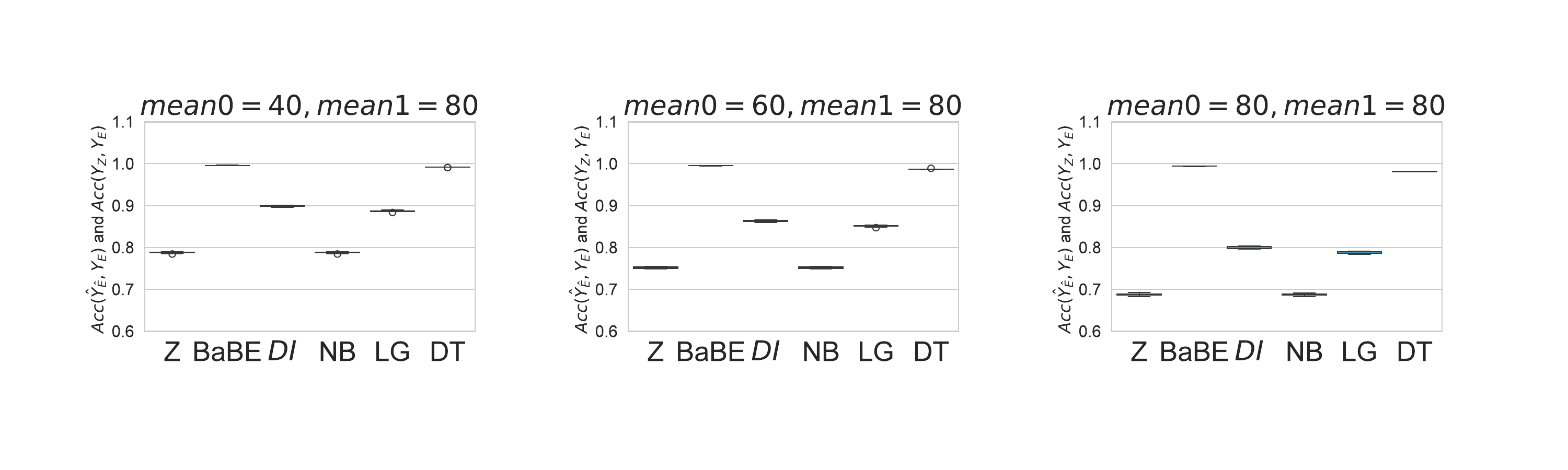}
     \vspace{-5mm}
     \caption{Experiments on the synthetic data sets: The accuracy between $\hat{Y}_{Z}$ and $Y_E$,  and between $\hat{Y}_{\hat{E}}$ and $Y_E$. }
\label{fig:Acc2}
\end{figure}

Figure~\ref{fig:ex2_diste} shows the distortion (Equation~\ref{eq:Dist}). BaBE again produces results that are closer to the true values of $E$ than the ones produced by other methods.

% \begin{figure}[H]
%     \centering
%     \includegraphics[width=\linewidth]{Babe/Graphs/Experiment3/3_dist_e1.pdf}
%     \caption{ The distortion for $S=1$.}
% \label{fig:ex2_dist1}
% \end{figure}

% \begin{figure}[H]
%     \centering
%     \includegraphics[width=\linewidth]{Babe/Graphs/Experiment3/3_dist_e0.pdf}
%     \caption{ The distortion for $S=0$.}
% \label{fig:ex2_dist0}
% \end{figure}

\begin{figure}[h]
    \centering
    \includegraphics[width=0.8\linewidth]{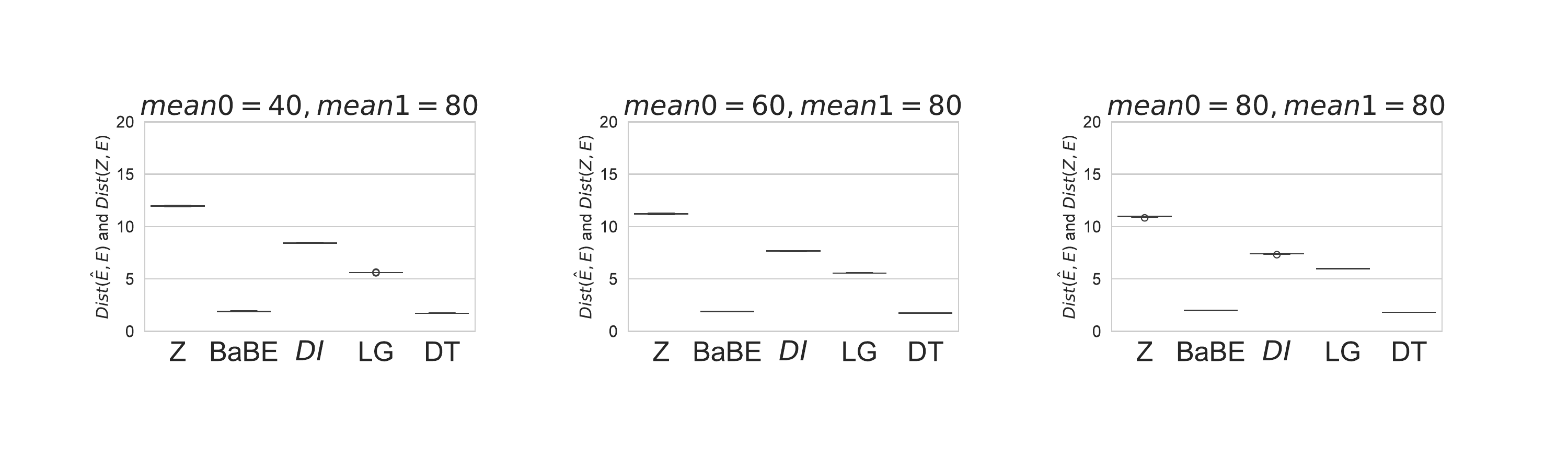}
     \vspace{-5mm}
     \caption{ Experiments on the synthetic data sets: The distortion between $Z$ and $E$ and between  $\hat{E}$ and $E$.}
\label{fig:ex2_diste}
\end{figure}

% \begin{figure*}[t]
%     \centering
%     \includegraphics[width=\linewidth]{Graphs/Experiment3/3_sp1.pdf}
%     \caption{ $\pr[ {Y}_{Z}=1|S=1]$ (for $Z$) and  $\pr[\hat{Y}_{\hat{E}}=1|S=1]$. }

% \end{figure*}

% \begin{figure*}[t]
%     \centering
%     \includegraphics[width=\linewidth]{Graphs/Experiment3/3_sp0.pdf}
%     \caption{ $\pr[ {Y}_{Z}=1|S=0]$ (for $Z$) and  $\pr[\hat{Y}_{\hat{E}}=1|S=0]$.  }

% \end{figure*}

% \begin{figure*}[t]
%     \centering
%     \includegraphics[width=\linewidth]{Graphs/Experiment3/3_spd.pdf}
%     \caption{Statistical Parity Difference (SPD). For  BaBE, DI and NB,  SPD is defined as $ \pr[\hat{Y}_{\hat{E}}=1|S=1] - \pr[\hat{Y}_{\hat{E}}=1|S=0]$. For $Z$,  the definition is similar, with $\hat{Y}_{\hat{E}}$ replaced by $Y_Z$}

% \end{figure*}

Figure~\ref{fig:Cond2} shows the conditional statistical parity difference on admission for each group, conditioned on $E$. The values for BaBE are close to zero, indicating the absence of discrimination. The DI method has decreased the discrimination with respect to $Z$. NB results are worse than the initial discrimination: it is possible that the non linear bias function together with the accuracy constraints inbuilt in the algorithm has impeded its performance. DT once again is a second best performer after BaBE.

% \begin{figure}[H]
%     \centering
%     \includegraphics[width=\linewidth]{Babe/Graphs/Experiment3/3_csp11.pdf}
%     \caption{ Experiment on the transfer of knowledge: $\pr[Y_Z=1|E=55,S=1]$ and $\pr[\hat{Y}_{\hat{E}}=1|E=55,S=1]$. }
% \label{fig:Cond2-1}
% \end{figure}

% \begin{figure}[H]
%     \centering
%     \includegraphics[width=\linewidth]{Babe/Graphs/Experiment3/3_csp01.pdf}
%     \caption{Experiment on the transfer of knowledge: $\pr[Y_Z=1|E=55,S=0]$ and $\pr[\hat{Y}_{\hat{E}}=1|E=55,S=0]$. }
% \label{fig:Cond2-0}
% \end{figure}

\begin{figure}[h]
    \centering
    \includegraphics[width=0.8\linewidth]{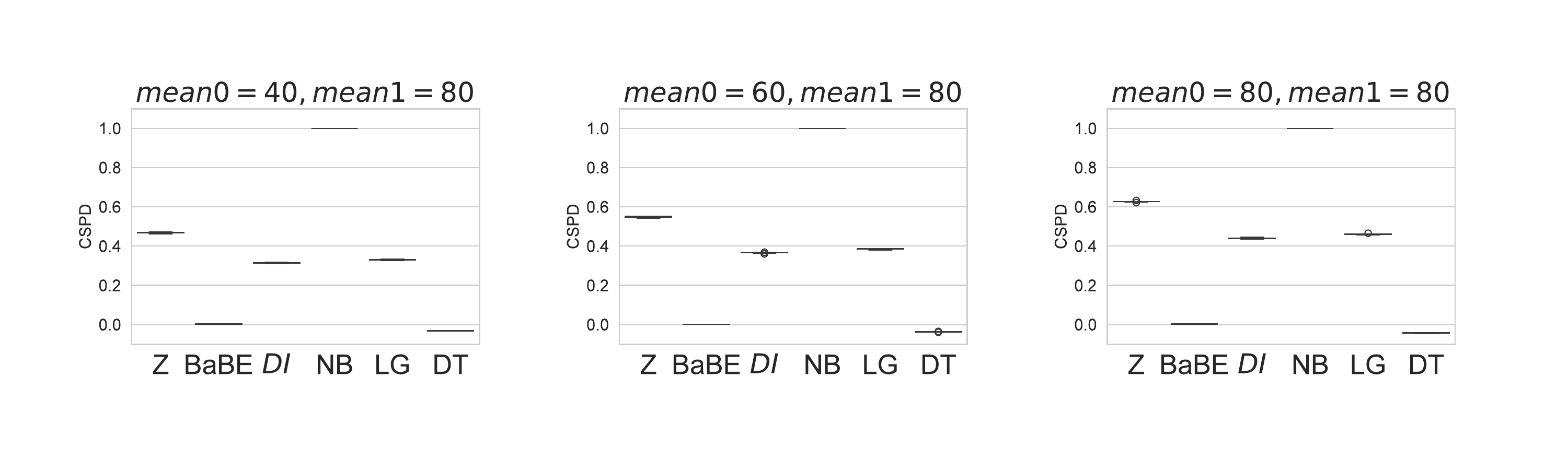}
     \vspace{-5mm}
     \caption{Experiments on the synthetic data sets: Conditional Statistical Parity Difference (CSPD). We recall that, for  BaBE, DI, NB, LG and  DT, the CSPD is defined as $\pr[\hat{Y}_{\hat{E}}=1|E,S=1]-\pr[\hat{Y}_{\hat{E}}=1|E,S=0]$. For $Z$,  the definition is similar, with $\hat{Y}_{\hat{E}}$ replaced by $Y_Z$.  }
\label{fig:Cond2}
\end{figure}

%%%%%%%%%%%%%%%%%%%%
 Figure~\ref{fig:EOD2bis} shows the probabilities of positive prediction when the true decision is positive, and the corresponding difference in equal opportunity. We note that the prediction based on $Z$ has a high probability to be positive for the group $1$, but not for the group $0$, therefore EOD for $Z$ is close to $1$. On the other hand, BaBE's prediction is based on the estimation of $E$, and hence tends to be equal to the true decision yielding EOD close to zero. Quite surprisingly, DI gives bad results, even though it is supposed to equalize the distributions for $S=0$ and $S=1$. However, DI decreases the mean for $S=1$ instead of increasing the mean for $S=0$, leaving the values for $S=0$ below the positive decision threshold ($80$). Similar considerations apply to NB and LG.

% \begin{figure}[H]
%     \centering
%     \includegraphics[width=\linewidth]{Babe/Graphs/Experiment3/3_eo1.pdf}
%     \caption{ Experiment on the transfer of knowledge:   $\pr[{Y}_{Z}=1|Y_E=1,S=1]$ and $\pr[\hat{Y}_{\hat{E}}=1|Y_E=1,S=1]$.  } 
%  \label{fig:EOD2-1}
% \end{figure}

% \begin{figure}[H]
%     \centering
%     \includegraphics[width=\linewidth]{Babe/Graphs/Experiment3/3_eo0.pdf}
%     \caption{Experiment on the transfer of knowledge:    $\pr[{Y}_{Z}=1|Y_E=1,S=0]$ and $\pr[\hat{Y}_{\hat{E}}=1|Y_E=1,S=0]$. }
% \label{fig:EOD2-0}
% \end{figure}

\begin{figure}[h]
    \centering
    \includegraphics[width=0.8\linewidth]{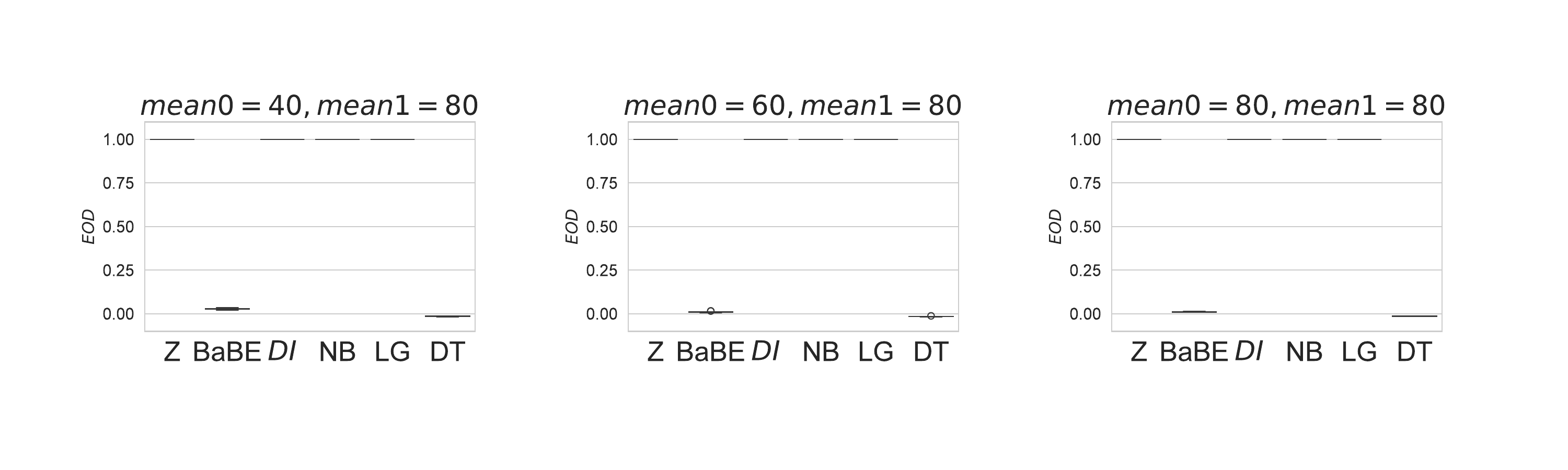}
     \vspace{-5mm}
    \caption{Experiments on the synthetic data sets:   Equal Opportunity Difference (EOD).
    We recall that, for  BaBE, DI, NB, LG and  DT, the EOD is defined as $ \pr[\hat{Y}_{\hat{E}}=1| Y_E = 1, S=1] - \pr[\hat{Y}_{\hat{E}}=1|Y_E = 1, S=0]$. For $Z$,  the definition is similar, with $\hat{Y}_{\hat{E}}$ replaced by $Y_Z$.}
  \label{fig:EOD2bis}
\end{figure}

Finally, Figure~\ref{fig:SPD3} compares the statistical difference (SPD) of the prediction $\hat{Y}_{\hat{E}}$  obtained with the various methods and the SPD of ${Y}_Z$. The SPD for $\hat{Y}_{\hat{E}}$ is defined in \eqref{eq:SPD}, for ${Y}_Z$ is defined as $\pr[Y_Z=1|S=1] - \pr[Y_Z=1|S=0]$.
When  $\mathit{mean0}=80$, that is, the same as $\mathit{mean1}$, BaBE achieves, correctly, $SPD=0$. In contrast, DI and NB  do not achieve equal distribution for S=1 and S=0, which is surprising since the algorithms are geared towards equality. We hypothesize that the performance of the algorithms is impeded by the non linear bias function and in-built accuracy constraints. The DT algorithm is closest to the performance of BaBE, as it is more suitable to handle non-linearity in the data set than other ML model (LG).

% \begin{figure}[H]
%     \centering
%     \includegraphics[width=\linewidth]{Babe/Graphs/Experiment3/3_sp1.pdf}
%     \caption{ $\pr[ {Y}_{Z}=1|S=1]$ (for $Z$) and  $\pr[\hat{Y}_{\hat{E}}=1|S=1]$ (for BaBE, DI and NB)}
%  \label{fig:SP31}
% \end{figure}

% \begin{figure}[H]
%     \centering
%     \includegraphics[width=\linewidth]{Babe/Graphs/Experiment3/3_sp0.pdf}
%     \caption{$\pr[ {Y}_{Z}=1|S=0]$ (for $Z$) and  $\pr[\hat{Y}_{\hat{E}}=1|S=0]$ (for BaBE, DI and NB)}
%  \label{fig:SP30}
% \end{figure}

\begin{figure}[h]
    \centering
    \includegraphics[width=0.8\linewidth]{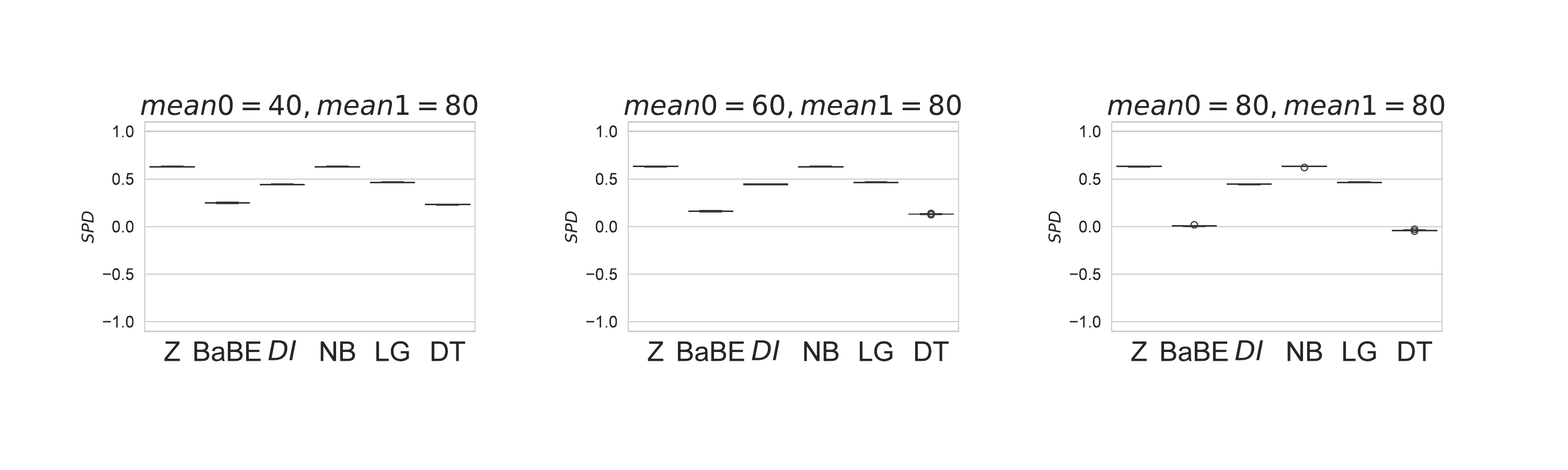}
    \vspace{-5mm}
    \caption{Experiments on the synthetic data sets: Statistical Parity Difference (SPD).
    We recall that, for  BaBE, DI, NB, LG and  DT, the SPD is defined as $ \pr[\hat{Y}_{\hat{E}}=1| S=1] - \pr[\hat{Y}_{\hat{E}}=1| S=0]$. For $Z$,  the definition is similar, with $\hat{Y}_{\hat{E}}$ replaced by $Y_Z$.}
    \label{fig:SPD3}
\end{figure}

% \subsubsection{Results on the NHANES data}
% \begin{figure}[H]
%     \centering
%     \includegraphics[width=\linewidth]{GRAPHS/density_s10.png}
%     \caption{Distributions of $E$ and $Z$ for $S=1$ (left) and $S=0$ (right) in NHANES data set.}
%   \label{fig:nhanesdist}
% \end{figure}

\subsection{The real-world data set}
The National Health and Nutrition Examination Survey (NHANES) \cite{NHANES} is a series of studies that are intended to evaluate the health and nutritional status of adults and children in the United States. The survey is unique in that it incorporates in-depth interviews and detailed physical examinations. Health-related questions and demographics are included in the NHANES interview.  For the survey, the sample was selected to represent the US population of all ages. To produce reliable statistics, NHANES oversamples individuals aged 60 and over, African Americans, and Hispanics.
NHANES is a popular source for studying biological aging~\cite{kwon2021toolkit,xu2023blunted,liu2023oxidative,nguyen2022biological}. The data set consists of 8243 samples. For our experiments, we use three variables from the data set, race (black or white), which is out $S$, chronological age (20-90), which is our $Z$, and an estimate of the biological age of the original KDM~\footnote{Klemera and Doubal's method for calculating the biological age from the set of biomarkers.} biological age (variable 'kdm0') which is our $E$. We choose chronological or biological age $75$ or more as the threshold to set $Y_Z=1$ and $Y_E=1$. This age group shows the most racial disparity in biological aging in the NHANES data set.
%(Figure~\ref{fig:agevseod})
Additionally, it is a reasonable age to check for age-related diseases or consider retirement.

Experiments on the NHANES data are carried out using Method 2 (cf. Section~\ref{method-2}) of the BaBE method. This is
because the conditional distribution of $Z|E,S$ does not allow the accurate estimation of every individual $E$. However, it still allows us to recover the aggregated distribution and estimate $\hat{Y}_{\hat{E}}$. 
In the experiments we consider only the fairness metric EODS, because the statistical disparity in the NHANES data is very small (owing to the oversampling of the minority population), so SPD is not interesting, and CSPD is not relevant because we apply Method 2.  

% ~\footnote{We also report intermediate steps for EOD: $\pr[\hat{Y}_{\hat{E}}=1|Y_E=1,S=1]$ and $ \pr[Y_Z=1|Y_E=1,S=1]$, $ \pr[\hat{Y}_{\hat{E}}=|Y_E=1,S=0]$  and $ \pr[Y_Z=1|Y_E=1,S=0]$} and  $\mathit{Acc}(\hat{Y}_{\hat{E}}, Y_E)$, $\mathit{Acc}(Y_{Z}, Y_E)$~\footnote{In addition we report intermittent steps to obtain accuracy measure: $\mathit{Acc}(\hat{Y}_{\hat{E}}|S=1, Y_E|S=1)$ and $\mathit{Acc}(Y_Z|S=1, Y_E|S=1)$, $\mathit{Acc}(\hat{Y}_{\hat{E}}|S=1, Y_E|S=0)$ and $\mathit{Acc}(Y_Z|S=1, Y_E|S=0)$}.

% We evaluated various metrics for the precision of the estimations and fairness and compared the performance of BaBE with disparate impact remover (DI) and with naive Bayes (NB).
% The results are illustrated in Figures \ref{fig:WD1} through \ref{fig:EOD1}. 
The boxplots are obtained by repeating the experiments ten times with the same parameters. We report the results for the values of $mean0$ equal to $40$, $60$ and $80$.

% \begin{figure}[H]
%     \centering
%     \includegraphics[width=0.5\linewidth]{GRAPHS/ageversuseod.png}
%     \caption{The graph shows the equal opportunity difference (EOD) between the $Y_E$ and $Y_Z$, when different thresholds for $Z$ (chronological age) are selected. The disparity is largest around the chronological age equal 75 years. }
%     \label{fig:agevseod}
% \end{figure}

% We perform the experiments on the first 40 samples applying the BaBE, the Disparate Impact Remover, and the Naive Bayes Pre-processing algorithms. 

% In Figure~\ref{fig:nhanes_dist} we report the distortion on $\hat{E}$ estimated by BaBE for each group. As we can see, it is very close to the true values of $E$. 

 Figure~\ref{fig:nhanes_ACC} shows the accuracy resulting from the application of BaBE, DI, and NB to the NHANES data set. As we can see, BaBE achieves better overall accuracy and significantly better accuracy for $S=1$.

 Figure~\ref{fig:nhanes_EOD} shows the equal opportunity from the application of BaBE, DI, and NB to the NHANES data set. BaBE achieves EOD close to zero. DI and NB preprocessing methods do not differ significantly from the estimated $EOD$ considering the original $Z$.

% \begin{figure}[h]\
%     \centering
%     \includegraphics[width=0.6\linewidth]{Babe/Graphs/bio_age/bio_dist_e.pdf}
%     \caption{Experiments on the NHANES data. Distortion.}
%     \label{fig:nhanes_dist}
% \end{figure}

\begin{figure}[h]
    \centering
    \includegraphics[width=0.8\linewidth]{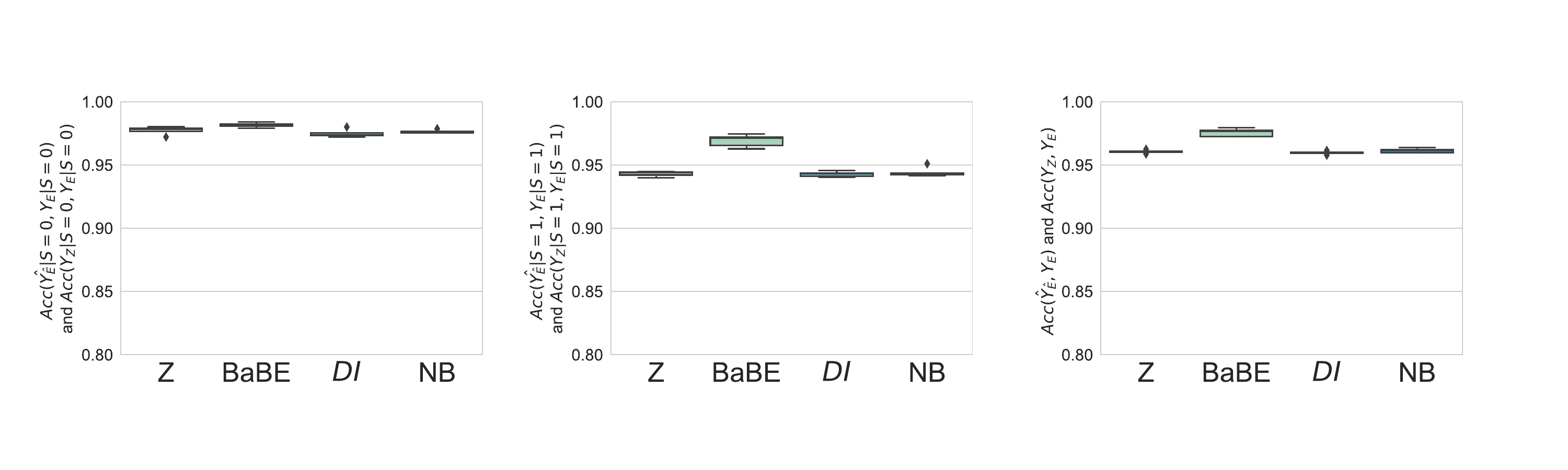}
     \vspace{-5mm}
     \caption{Experiments on the NHANES data. The Accuracy  for the two groups separately, and overall.}
 \label{fig:nhanes_ACC}
\end{figure}

% \begin{figure}[H]
%     \centering
%     \includegraphics[width=0.6\linewidth]{Babe/Graphs/bio_age/bio_spd.pdf}
%     \caption{Experiments on the NHANES data. $SPD$.}
%     \label{fig:nhanes_SPD}
% \end{figure}

% \begin{figure}[H]
%     \centering
%     \includegraphics[width=0.6\linewidth]{Babe/Graphs/bio_age/bio_cspe.pdf}
%     \caption{Experiments on the NHANES data. $CSPD$.}
%     \label{fig:nhanes_CSP}
% \end{figure}

\begin{figure}[h]
    \centering
    \includegraphics[width=0.8\linewidth]{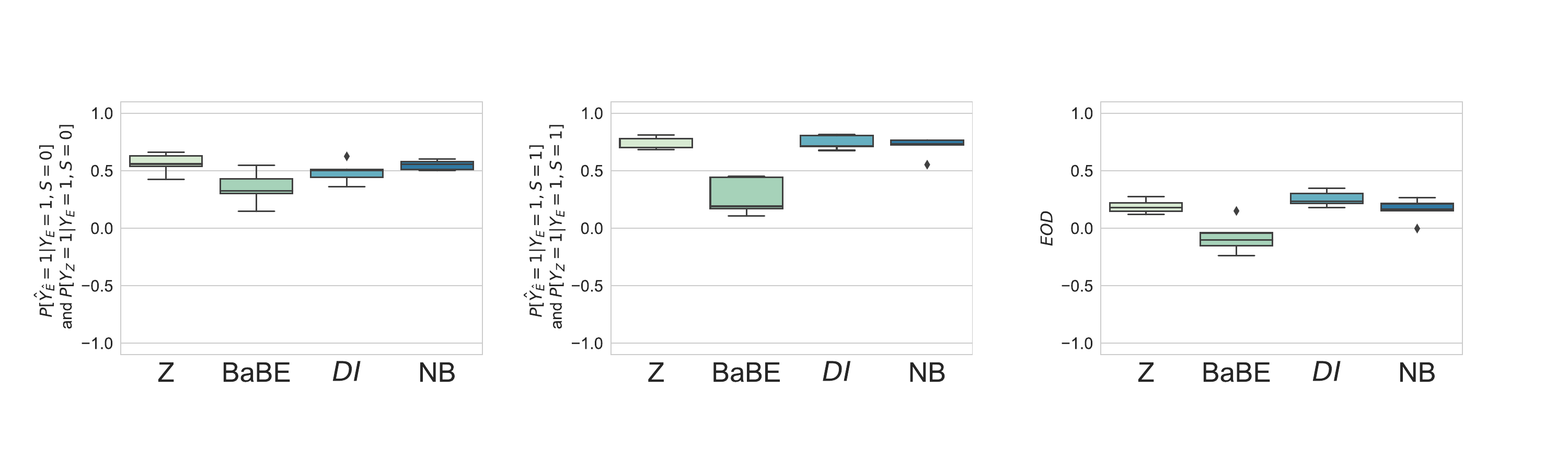}
     \vspace{-5mm}
     \caption{Experiments on the NHANES data. Equal opportunity difference (EOD), for the two groups separately, and overall.}
    \label{fig:nhanes_EOD}
\end{figure}

% \begin{figure}[H]
%     \centering
%     \includegraphics[width=0.6\linewidth]{Babe/Graphs/bio_age/bio_w.pdf}
%     \caption{Experiments on the NHANES data. The Wasserstein distance.}
%     \label{fig:nhanes_W}
% \end{figure}

\subsection{Discussion}

Our experiments show that BaBE performs well for the fairness notions for which BaBE is designed, i.e., CSPD  and EOD, while maintaining good accuracy. 

BaBE performs well also when $\pr[E|S]$   is different from that of the data in which $\pr[Z|E,S]$ has been computed (Figure~\ref{fig:Wass2}), which shows that BaBE is compatible with the transfer of causal knowledge to populations with different distributions. 
On the contrary, DI and NB highly depend on the distribution as they always aim to minimize SPD. Note that minimizing SPD in the NHANES data set would still result in discrimination against black people, who on average have higher biological age than white people of the same chronological age. 

The results of machine learning algorithms LG and DT show the sensitivity to the change in distribution of $E|S$. For example, the accuracy with respect to $Y_E$ of LG and DT is highest on the data set 1, where $E|S$ is the same as in the training data ($mean0=40, mean1=80$) and degrades in the data sets 2 and 3, where it is different. In addition, LG performs worse the DT in all experiments. This is expected, because the relationship between $Z$ and $E$ is non linear. We note that BaBE is able to recover $\hat{E}$ without restrictions on the functional relationship in the data set. We acknowledge, that it is possible that a more complex machine learning algorithm could perform better on the proposed data set than linear regression, however it would imply higher computational cost. Moreover, it might still be affected by the distribution shift~\cite{ovadia2019can}. 
% This is an important result, because the estimation of the proportion is itself error-prone, as the initial estimation by the authors was $47\%$, and 
% %(we assume, that $E$ is not observed). 
% it was later corrected in the official repository~\footnote{Public Release Data and Code \url{https://gitlab.com/labsysmed/dissecting-bias/-/blob/master/README.md}}. 

It is important to mention that the performance of BaBE is dependent on the
invertibility of $\pr[Z|E,S=s]$ (seen as stochastic matrix, aka bias matrix), because invertibility is necessary for the uniqueness of the MLE. However, even when the matrix is not invertible, we are able to obtain favorable results.  Indeed, in all our experiments the bias matrices we produce from the synthetic data are not invertible, to mimic the more realistic scenarios. 
Preliminary experiments show that the diagonal deterministic matrix produces the highest precision for the estimation of distributions $\pr[E|S]$, and highest accuracy of the prediction $\hat{Y}_{\hat{E}}$. 
We leave a more systematic study on how precision and accuracy depend on  $\pr[Z|E,S=s]$ as a topic for future work.  
\section{Conclusions and Future work}

We have proposed BaBE, a framework to use knowledge of a biasing mechanism from domain-specific studies to perform data pre-processing, aiming at achieving Conditional Statistical Parity and Equal Opportunity when the explaining variable, and, consequently, the true decision, are not contained in the data. The BaBE algorithm uses the bias mechanism to estimate
the probability distributions of the explaining variable, and it performs equally well even when the population distributions are different from the ones in 
which the study of the bias was conducted. A distinguishing feature of our approach is that \emph{ we do not need to assume that the explaining variable is independent of the sensitive attribute}. 
One challenging direction for future work is to explore how the precision of the estimation, the accuracy of the prediction, and the fairness level depend on the form of the matrices $\hat \pr[E|Z,S]$, and how the latter depends on the matrices representing the external knowledge (i.e., the bias mechanism)  $\hat\pr[Z|E,S]$. 
% Another promising direction is exploring the application of BaBE for fairness-related distribution shift auditing.
%
% Moreover, we plan to collaborate with domain experts to define formal conditions for measuring the $\hat \pr[Z|E,S]$ matrices accurately. 

We trust our method to serve as a tool to enhance interdisciplinary collaboration between domain experts and ML Fairness practitioners. 

\begin{acks}

This work was supported by the European Research Council (ERC) project HYPATIA under the European Union’s Horizon $2020$ research and innovation programme, grant agreement $n. 835294$. The work of Ruta Binkyte was also supported by Bundesministeriums fur Bildung und Forschung (PriSyn), grant
$No. 16KISAO29K$.

\end{acks}

\bibliographystyle{ACM-Reference-Format}
\bibliography{main}

\appendix
\setcounter{figure}{0}
\section{ Derivation of BaBE as an instance of the EM method}\label{sec:appendix}

\label{appendix}

In this section, we show how to apply the EM method to the problem we are considering, thus obtaining the main algorithm of our method BaBE. 

\newcommand{\EX}{{\mathbb E}}
\newcommand{\define}{\stackrel{\text{def}} =}

Let $E$, $Z$  and $S$ be random variables on $\mathcal E$, $\mathcal Z$  and $\mathcal S$, with generic elements $e,z$ and $s$
respectively. 
%For a fixed $s \in \S$, we denote by $\bar{e} = e_1, e_2, \ldots , e_N $ a sequence of $N$ i.i.d. samples from $P[E|S=s]$. Correspondingly, we denote by $\bar{z} = z_1, z_2, \ldots , z_N$ the sequence of perturbed version of $e_i$, where the biasing mechanism that yields $z$ from $e$ and $s$ is denoted by $P [Z = z|E = e, S = s]$$. 
Let $(\bar{z},\bar{s}) \ = \{(z_i,s_i) \ |\  i = [1,\ldots,N]\}$ be a sequence of samples from the joint distribution $P[Z,S]$, let 
\begin{equation}
\label{eq:defHatP}
\bar z_s \define  \{ z_i\ :\ i \in \{1,...,N\} \wedge s_i = s \}  
\end{equation}
be the subsequence of $\bar z$ of elements paired with $s$ in the samples
and let $M$ be $|\bar z_s|$.
Then, the empirical probability of $Z=z$ given $S=s$ (i.e., the frequency of $z$ in the samples with $S=s$) is defined as:
\begin{equation}
% \label{eq:defHatP}
\varphi_s[z, \bar z_s] \define \frac{  | \{ z_i \in \bar z_s\ :\ z_i = z \}  |  } { M }.
\end{equation}

% $P_X[x] = P[X=x|U=u]$, where $U$ is a latent unfair confounder causing the bias in $X$. $U$ is considered fixed: $P_X[x] = P_{X|U=u}$ and $P_{Z|X} = P_{Z|X, U=u}$ and estimated for each value of $U=u$.\\

Now, given $(\bar{z},\bar{s})$, $s \in {\mathcal S}$,  $\varphi_s[z,\bar z_s]$ and the conditional distribution $P[Z|E,S]$, we want to estimate the (unknown) $P[E|S]$ by applying the Expectation-Maximization (EM) method, i.e., by finding the probability distribution on $\mathcal E$ that maximizes the probability of observing $\bar z_s$ given $s$ (and that therefore is the best explanation of what we have observed). 
More precisely, we want to prove that our algorithm yields a Maximum Likelihood Estimation (MLE) $\hat{P}[E|S]$ that approximates $P[E|S]$. To this end, let $\Theta$ denote the set of all distributions on $\mathcal E$ conditioned on $S=s$, and let $\theta$ range over it. 
The {\em log-likelihood function} for $\bar z_s$ is $L_{\bar z_s} : \Theta \rightarrow {\mathbf R}$ such that
\begin{equation}
L_{\bar z_s}(\theta) \define \log P[\bar{Z}_s=\bar{z}_s|\theta]
\end{equation}
where $\bar Z_s$ denotes a sequence of $M$ random samples drawn from $\mathcal Z$ when $S=s$.
Given $\bar z_s$, a MLE of the unknown $P[E|S]$ is then defined as $\argmax_\theta L_{\bar z_s}(\theta)$, i.e., as the $\theta$ that maximizes $L_{\bar z_s}(\theta)$ (and therefore $P[\bar{Z}_s=\bar{z}_s|\theta]$, since $\log$ is monotone). 

%The EM framework is a powerful method used when the  model has hidden data that, if known, would make the estimation procedure easier. 
We now show how to adapt the EM framework to the above setting.
% , by following the proof of \cite{agrawal_ibu}.
% Since both $P_{E|S}$ and $\bar E$ are actually unknown, 
%instead of $\log P_{\bar E | \bar Z,S,P_{E|S}} (\bar e | \bar z,s)$, we consider its expected value computed by using a prior approximation $\theta'$ of $P_{E|S}$. This expectation yields 
We start by defining the  function 
\begin{equation}
Q(\theta,\theta') \define \EX[\log \bar\theta\ |\ \bar{Z}_s=\bar{z}_s, S=s, \theta'] 
\end{equation} 
where $\bar\theta$ denotes the probability distribution on sequences $\bar{e} = e_1, e_2, \ldots , e_M $ of i.i.d. events all with probability distribution $\theta$.  
The above expectation is taken  on all $\mathcal E$ and conditioned on $\bar{Z}_s=\bar{z}_s, S=s$, and assuming $\theta'$ as a prior approximation of $P[E|S]$.

The function $Q$ has the nice property that $L_{\bar z_s}(\theta)- L_{\bar z_s}(\theta')\geq Q(\theta,\theta')- Q(\theta',\theta')$. Hence, in order to improve the approximation of the MLE, i.e.,  to find an  estimation $\theta$ that improves the  estimation $\theta'$, it is sufficient to compute $Q(\theta,\theta') $ and find the $\theta$ that maximizes it.

% The following two Lemmata are used to show that function $Q$ has the nice property that the value that maximizes it also maximizes the log-likelihood function.

% The first lemma will be used to state that $Q$ is strictly concave. This fact, together with the assumption that $P_{Z|E,S}$ seen as a stochastic matrix is invertible, will ensure us that $L_{\bar z}(\cdot)$ has a unique maximum.

\begin{lemma}
\label{lem:Q}
\begin{align*}
\textstyle
Q(\theta,\theta')
=\sum\limits_{i=1}^{M}\sum\limits_{e \in {\mathcal E}}\frac{P[Z_s=z_i|E=e,S=s]\ \theta'[e|s]}{\sum\limits_{e' \in {\mathcal E}}P[Z_s=z_i|E=e',S=s]\ \theta'[e'|s]}\log \theta[e|s].
\end{align*}
\end{lemma}

Given that the $E_i$s are i.i.d., by definition and linearity of conditional expectation, we have that:

\begin{align}
\label{eq:YY}
& \EX [\log \bar\theta\ |\ \bar{Z}_s=\bar{z}_s, S=s, \theta'] \nonumber \\
& = \EX \left.\left[\log \prod\limits_{i=1}^{M} \theta[e_i|s]\ \right |\ \bar{Z}_s=\bar{z}_s, S=s, \theta' \right]  \nonumber \\
& =\EX \left.\left[\sum\limits_{i=1}^{M} \log \theta[e_i|s]\ \right |\ \bar{Z}_s=\bar{z}_s, S=s, \theta' \right]  \nonumber \\
& = \sum\limits_{i=1}^{M} \EX [\log \theta[e_i|s]\ |\ \bar{Z}_s=\bar{z}_s, S=s, \theta'] \nonumber \\
& = \sum\limits_{i=1}^{M}\sum\limits_{e \in {\mathcal E}} P[E=e|Z_s=z_i,S=s] \ \log \theta[e|s]
\end{align}
where $P[E|Z_s,S]$ is a probability based on the estimation  $\theta'$ of the unknown $P[E|S]$.
By taking the marginal distribution, we have that:

\begin{align}\label{eq:marginal}
& P[Z_s=z_i|S=s] \nonumber\\
& =\sum\limits_{e' \in {\mathcal E}} P[Z_s=z_i,E=e'|S=s]\nonumber\\
& = \sum\limits_{e' \in {\mathcal E}} P[Z_s=z_i|E=e',S=s]\theta'[e'|s].
\end{align}

%where the last equality holds since we are assuming $E$ and $S$ to be independent (so, $\theta'[\,\cdot\,|s] = \theta'[\,\cdot\,]$).
%For the same reason, 
By the conditional Bayes theorem and \eqref{eq:marginal}, we have that

\begin{align*}
\label{eq:XX}
&P[E=e|Z_s=z_i,S=s]   \\
&=\frac{P[Z_s=z_i|E=e,S=s] \theta'[e|s]}{P[Z_s=z_i|S=s]} \\
&= \frac{P[Z_s=z_i|E=e,S=s] \theta'[e|s]}{\sum\limits_{e' \in \mathcal E} P[Z_s=z_i|E=e',S=s] \theta'[e'|s]}
\end{align*}

By plugging 
%\eqref{eq:XX} 
the latter equality into \eqref{eq:YY}, we conclude the proof.

%\textcolor{blue}{SOMETHING LIKE WHAT FOLLOWS HAS BEEN SAID IN SECT.2... KEEP OR REMOVE FROM HERE?}

%By recalling the well-known notion of Kullback–Leibler divergence $D_{KL}(P||Q) \triangleq \sum\limits_{i=1}^N p_i \log\frac{p_i}{q_i}$, we have that
%\begin{equation}
%L_{\bar z}(\theta) = Q(\theta,\theta') + D_{KL}(\theta||\theta')
%\end{equation}
%This fact, together with $D_{KL}(\theta||\theta') \geq D_{KL}(\theta'||\theta')$ (Gibbs' inequality), gives the following fundamental property of every EM algorithm:
%\begin{equation}
%\forall \theta, \theta' \in {\Theta}\,.\ L_{\bar z}(\theta) - L_{\bar z}(\theta') \geq Q(\theta,\theta') - Q(\theta'|\theta')    
%\end{equation}
%This amounts to say that, if $\theta$ is chosen to improve $Q(\theta,\theta')$ w.r.t. $Q(\theta'|\theta')$, then $L_{\bar z}(\theta)$ is also improved w.r.t. $L_{\bar z}(\theta')$ by at least the same amount. Therefore, $L_{\bar z}(\,\cdot\,)$ monotonically grows by taking a $\theta$ that improves the previous approximation $\theta'$; since $L_{\bar z}(\,\cdot\,)$ is bounded from above by 0, it must converge (by the monotone convergence theorem). 

The next Lemma tells us that  $\hat{P}[E|S]^{(t+1)}$ (as defined in %\eqref{eq:iterative-step} and in
%$\sum\limits_{z \in {\mathcal Z}}\varphi_s[z, \bar z] \frac{P_{Z|E,S}[z|e,s]\hat P_E^{(t)}[e]}{\sum\limits_{e'\in {\mathcal E}}P_{Z|E,S}[z|e',s]\hat P_E^{(t)}[e']}$ 
Algorithm~1) is the distribution that maximizes $Q\left(\ \cdot\ , \hat P[E|S]^{(t)}\right)$.
% this fact, together with the previous argument, will allow us to conclude that the algorithm approximates the (unique) MLE $\hat P_{E|S}$ (see Theorem \ref{thm:soundness} later on).
 This fact
%  , together with the uniqueness of the maximum of $L_{\bar z}(\cdot)$, 
will allow us to conclude that the algorithm approximates the  MLE $\argmax_\theta L_{\bar z_s}(\theta)$.
%  (see Theorem \ref{thm:soundness} later on).

\begin{lemma} The $\theta$ that maximizes
$Q(\ \cdot\ , \theta')$ is such that, for every $e \in {\mathcal E}$:
$$
\theta[e|s] = \sum\limits_{z \in {\mathcal Z}}\varphi_s[z, \bar z_s]\frac{P[Z_s=z|E=e,S=s]\ \theta'[e|s]}{\sum\limits_{e'\in {\mathcal E}}P[Z_s=z|E=e',S=s]\ \theta'[e'|s]}.
$$
\end{lemma}
\begin{proof}
By the method of Lagrangian multipliers, we can find the $\theta$ that maximizes $Q(\theta, \theta')$ by adding to the latter the term $\lambda\left(\sum\limits_{e \in {\mathcal E}}\theta[e|s]-1\right)$, for some $\lambda$, and study the function
\begin{equation}
F(\theta, \theta') \triangleq Q(\theta, \theta') + \lambda\left(\sum\limits_{e\in {\mathcal E}}\theta[e|s]-1\right)
\end{equation}
that has the same stationary points as $Q(\theta, \theta')$ since $\sum\limits_{e \in {\mathcal E}}\theta[e|s] = 1$, being $\theta$ a probability distribution on $\mathcal E$ given $S=s$.
To find the stationary points of $F$, we impose that all its partial derivatives, including the one w.r.t. $\lambda$, are equal to $0$. 
For the latter one, we require that
\begin{equation}
\label{eq:derLambda}
\frac{\partial F}{\partial \lambda} = \sum\limits_{e\in {\mathcal E}}\theta[e|s]-1=0 
%\iff \sum\limits_{e\in {\mathcal E}}P_E[e] =1
\end{equation}
and this trivially holds since $\theta[\cdot|s]$ is a distribution for every $s$.
For the former ones, by relying on Lemma \ref{lem:Q}, we impose that, for every $e \in {\mathcal E}$,

\begin{align}\label{eq:derPE}
&\frac{\partial F}{\partial \theta[e|s]} \nonumber \\
&= \frac{1}{\theta[e|s]}\sum\limits_{i=1}^{M}
\frac{P[Z_s=z_i|E=e,S=s]\ \theta'[e|s]}{\sum\limits_{e'\in {\mathcal E}}P[Z_s=z_i|E=e',S=s]\ \theta'[e'|s]}+ \lambda
\nonumber\\
& = 0
\end{align}

By multiplying the last equality by $\theta[e|s]$, we get:
\begin{equation}
\label{eq:bla}
 \lambda \,\theta[e|s] = - \sum\limits_{i=1}^{M} \frac{P[Z_s=z_i|E=e,S=s]\ \theta'[e|s]}{\sum\limits_{e'\in {\mathcal E}}P[Z_s=z_i|E=e',S=s]\ \theta'[e'|s]}.
\end{equation}
By summing both sides of \eqref{eq:bla} on all $e\in\mathcal E$, we obtain:
\begin{align}
\label{eq:boh}
&\lambda \sum\limits_{e \in {\mathcal E}}\theta[e|s] 
\nonumber\\
&=  - \sum\limits_{e \in {\mathcal E}}\sum\limits_{i=1}^{M}\frac{P[Z_s=z_i|E=e,S=s]\ \theta'[e|s]}{\sum\limits_{e' \in {\mathcal E}}P[Z_s=z_i|E=e',S=s]\ \theta'[e'|s]} \nonumber\\ 
&= - \sum\limits_{e \in {\mathcal E}} \sum\limits_{z \in {\mathcal Z}}\varphi_s[z, \bar z_s] M \frac{P[Z_s=z|E=e,S=s]\ \theta'[e|s]}{\sum\limits_{e' \in {\mathcal E}}P[Z_s=z|E=e',S=s]\ \theta'[e'|s]} \nonumber\\
&= -M \sum\limits_{z \in {\mathcal Z}} \varphi_s[z, \bar z_s] \frac{\sum\limits_{e \in {\mathcal E}}P[Z_s=z|E=e,S=s]\ \theta'[e|s]}{\sum\limits_{e' \in {\mathcal E}}P[Z_s=z|E=e',S=s]\ \theta'[e'|s]} \nonumber\\
&= -M \sum\limits_{z \in {\mathcal Z}} \varphi_s[z, \bar z_s] \nonumber\\
&= -M
\end{align}
where the last step holds because of \eqref{eq:defHatP}, and
the second step holds because, again by \eqref{eq:defHatP}, we have that, for any function $f$:
\begin{equation}
\label{eq:f}
\sum\limits_{i=1}^M f(z_i) = \sum\limits_{z \in {\mathcal Z}} \varphi_s[z, \bar z_s] M f(z).
\end{equation}

%; in particular, by \eqref{eq:defHatP}, $\hat P_Z[z] \neq 0$ only for the $z_i$'s.
Hence, since $\theta[\cdot|s]$ is a probability distribution on $\mathcal E$, we obtain that \eqref{eq:derPE} is satisfied by taking $\lambda = -M$. 

Therefore, by isolating $\theta[e|s]$ from \eqref{eq:bla} and by using \eqref{eq:f}, we can conclude that, for every $e \in {\mathcal E}$, we have that
\begin{align*}
 & \theta[e|s] \\
& = - \frac 1 \lambda \sum\limits_{i=1}^{M} \frac{P[Z_s=z_i|E=e,S=s]\ \theta'[e|s]}{\sum\limits_{e'\in {\mathcal E}}P[Z_s=z_i|E=e',S=s]\ \theta'[e'|s]}
 \\
 &=\frac 1 M \sum\limits_{i=1}^{M} \frac{P[Z_s=z_i|E=e,S=s]\ \theta'[e|s]}{\sum\limits_{e'\in {\mathcal E}}P[Z_s=z_i|E=e',S=s]\ \theta'[e'|s]}
 \\
 &=\sum\limits_{z \in {\mathcal Z}} \varphi_s[z, \bar z_s] \frac{P[Z_s=z|E=e,S=s]\ \theta'[e|s]}{\sum\limits_{e'\in {\mathcal E}}P[Z_s=z|E=e',S=s]\ \theta'[e'|s]}
\end{align*}
% To conclude, notice that $\theta$ is the maximum of $Q(\ \cdot\ , \theta')$. Indeed, 
% like in \cite[Prop.\,4.1]{agrawal_ibu}, we can show that $\Theta$ is convex and that $Q$ is strictly concave (by Lemma \ref{lem:Q}, the latter function is the linear combination of strictly concave functions); this implies that it has a unique maximum over $\Theta$ (viz., $\theta$).
\end{proof}
Now, for the given $s\in \mathcal S$, we define the sequence $\left \{\ \hat{P}[E|S=s]^{(t)}\, \right \}_{t\geq 0}$ as follows: 
\begin{eqnarray*}
   \hat{P}[E|S=s]^{(0)}  &\define  &\mbox{any fully supported distribution}\\[1em]
   \hat{P}[E|S=s]^{(t+1)} &\define &\argmax_\theta Q(\theta, \hat{P}[E|S=s]^{(t)})
\end{eqnarray*}

The next theorem states the key property of our algorithm, i.e. that $\left \{\ \hat{P}[E|S=s]^{(t)}\, \right\}_{t\geq 0}$  tends to the MLE $\argmax_\theta L_{\bar z_s}(\theta)$. The proof of the theorem follows from the fact that $Q(\theta , \theta')$ has  continuous derivatives in both its arguments, and from Theorem 4.3 in \cite{agrawal_ibu} (which is a reformulation of a result due to Wu \cite{Wu:83:jastat}). 

\begin{theorem}
\label{thm:soundness}
$\underset{t \rightarrow \infty }{\lim} \hat{P}[E|S=s]^{(t)} \ =\  \underset{\theta}\argmax\   L_{\bar z_s}(\,\theta\,)$.
\end{theorem}
% \begin{proof}
% We note that $Q(\theta , \theta')$ has derivatives continuous in both its arguments; by using the formulation of a result due to Wu \cite{Wu83} provided in \cite[Thm.\,4.3]{agrawal_ibu}, this implies that the sequence $\left\{\hat P_{E|S}^{(t)}\right\}_{t \geq 0}$ converges to $\argmax  L_{\bar z}(\,\theta\,)$, as required.
% \end{proof}

Furthermore, if $P[Z|E,S]$,  seen as a stochastic matrix, is invertible, then the MLE  $\underset{\theta}\argmax \  L_{\bar z_s}(\,\theta\,)$ is  unique. The proof follows from Theorem 4 in \cite{elsalamouny2020generalized}.

\section{ Execution metrics for BaBE Algorithm on Synthetic Data}\label{sec:appendix2}

Here we provide the number of iterations and execution time for the BaBE algorithm on the synthetic data of Section ~\ref{sec:experiments}. The code was run on a Apple M3 Pro with 12 cores 36 GB of memory. No GPU was used.

\begin{figure}[h]
    \centering
    \includegraphics[width=0.8\linewidth]{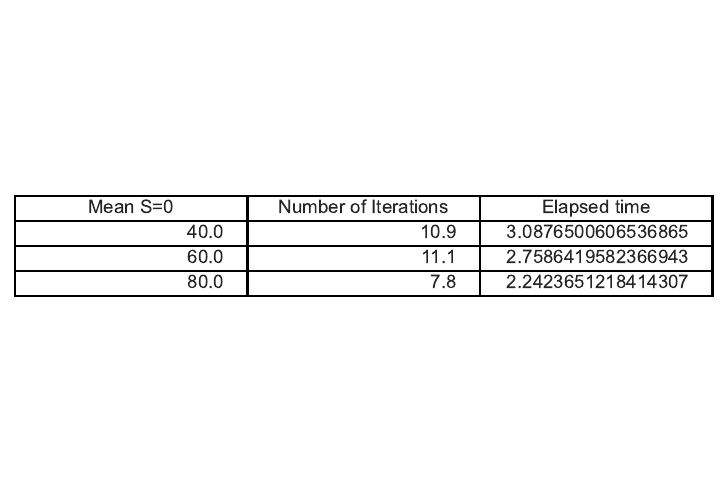}
     \vspace{-5mm}
     \caption{Execution time (in seconds) and an average number of iterations for BaBE algorithm, for each group of data sets where the mean for S=0 is varied.}
    \label{fig:nhanes_EOD}
\end{figure}

\end{document}